\documentclass{svjour3}
\journalname{Machine Learning}
\usepackage{natbib}
\usepackage{tikz}
\usepackage{latexsym}
\usepackage{amssymb}
\usepackage{amsmath}
\usepackage{subcaption}
\usepackage{tcolorbox}
\usepackage{appendix}
\usetikzlibrary{shapes,trees,arrows}
\usetikzlibrary{calc,through}
\usetikzlibrary{automata}
\usetikzlibrary{positioning}

\newtheorem{mydefinition}{Definition}
\newtheorem{myexample}{Example}
\newtheorem{myremark}{Remark}
\newtheorem{mylemma}{Lemma}

\newtheorem{myproposition}[mylemma]{Proposition}
\def\qedsymbol{$\blacksquare$}
\def\qed {{
\parfillskip=0pt        
\widowpenalty=10000     
\displaywidowpenalty=10000  
\finalhyphendemerits=0  
\leavevmode             
\unskip                 
\nobreak                
\hfil                   
\penalty50              
\hskip.2em              
\null                   
\hfill                  
\qedsymbol
\par}} 
\newcommand{\PEX}{{\mathtt{PEX}}}
\newcommand{\PXP}{{\mathtt{PXP}}}
\newcommand{\PXPk}{{\mathtt{PXP(k)}}}
\newcommand{\Ratify}{{\mathit{RATIFY}}}
\newcommand{\Revise}{{\mathit{REVISE}}}
\newcommand{\Refute}{{\mathit{REFUTE}}}
\newcommand{\Reject}{{\mathit{REJECT}}}
\newcommand{\Term}{{\mathit{TERM}}}
\newcommand{\Init}{{\mathit{INIT}}}
\newcommand{\Match}{{\mathtt{MATCH}}}
\newcommand{\Predict}{{\mathtt{PREDICT}}}
\newcommand{\Agree}{{\mathtt{AGREE}}}
\newcommand{\Learn}{{\mathtt{LEARN}}}

\begin{document}
\title{A Model for Intelligible Interaction Between Agents That Predict and Explain}
\author{A. Baskar \and 
        Ashwin Srinivasan \and 
        Michael Bain \and
        Enrico Coiera
}
\authorrunning{Baskar, Srinivasan, Bain, Coiera} 
\institute{A. Baskar \at 
				Dept. of Computer Science \& Information Systems\\
				BITS Pilani, Goa Campus, Goa. 					\email{abaskar@goa.bits-pilani.ac.in} \and
    Ashwin Srinivasan \at 
				Dept. of Computer Science \& Information Systems and APPCAIR \\
            BITS Pilani, Goa Campus, Goa.
						\email{ashwin@goa.bits-pilani.ac.in}
            \and
            Michael Bain \at
              School of Computer Science and Engineering \\
							University of New South Wales, Sydney.
							\email{m.bain@unsw.edu.au}
		    \and
		     Enrico Coiera \at
  Centre for Health Informatics, Macquarie University\\
 Sydney, Australia. 
 \email{enrico.coiera@mq.edu.au} 
		    }

\maketitle
\begin{abstract}%
Machine Learning (ML) has emerged as a powerful form of data modelling with 
widespread applicability beyond its roots in the design of autonomous agents.
However, relatively little attention has been paid to the interaction between
people and ML systems. In this paper we view interaction between humans and
ML systems within the broader context of communication between agents capable
of prediction and explanation. We formalise the interaction model by taking 
agents to be automata with some special characteristics and define a protocol
for communication between such agents. We define One- and Two-Way Intelligibility
as properties that emerge at run-time by execution of the protocol. 
The formalisation allows us to identify conditions under which run-time sequences
are bounded, and identify conditions under which the protocol can correctly implement
an axiomatic specification of intelligible interaction between a human and an ML system.
We also demonstrate using the formal model to: (a) identify instances
of One- and Two-Way Intelligibility in literature reports on humans interacting
with ML systems providing logic-based explanations, as is done in
Inductive Logic Programming (ILP); and (b) map interactions
between humans and machines in an elaborate natural-language based dialogue-model to
One- or Two-Way Intelligible interactions in the formal model.             
\end{abstract}
\section{Introduction}

The need for predictions made by machine-constructed models to be intelligible to a human has
been evident for at least four decades. To the best of our knowledge, the earliest identification
of a possible mismatch in the representations used by humans and machines was by Michie in~\citep{michie:window82}. 
He describes the notion of a `Human Window' of comprehension based on
constraints on computation and storage constraints imposed by the biology of the brain. 
Consequences of machine-constructed assistance falling outside this human window are examined
on synthetic problems (chess endgames) in~\citep{kopec:thesis}, who also describe some real-life 
disasters arising from the use of machine-constructed assistance for humans operating in safety-critical
areas (the Three Mile Island reactor meltdown being one such). 
Assuming the existence of the human window, Michie went on to propose a classification of machine-learning
(ML) systems into three categories~\citep{michie:ewsl88}.
Weak ML systems are concerned only with improving performance, given
sample data. Strong ML systems improve performance, but are also required to communicate what it
has learned in some human-comprehensible form (Michie assumes this will be symbolic). 
Ultra-strong ML systems are Strong ML systems that can also teach the human to improve his or her
performance. 
This categorisation has recently informed a similar 3-way categorisation
for the use of AI tools in scientific discovery~\citep{krenn:nature2022}, and
to evaluate an Inductive Logic Programming (ILP) as a form of Ultra-Strong
Machine Learning~\citep{mugg:expl}.

The following aspects of Michie's characterisation are worth emphasising.
Firstly, it is clearly intended  for use in a human-in-the-loop
setting, though the human can be a teacher, student or collaborator. Secondly, the characterisation
is about intelligibility, not intelligence. Intelligibility as stated is a relation
between the ML system (the sender), what the ML system communicates (the message), and the human
(the receiver). Thus, the ML system can employ any representation for its internal model; all that
is needed is that it can communicate the ``how'' and ``why'' in a form that lies within the human
window of comprehension.\footnote{The adjective `weak' in the first category does not mean the
ML engine's performance is poor. It simply indicates that the constraints on the learner mean that it is not
required to communicate its update to the human.} Thirdly, it appears to be a classification of an ML system based on 
one-way communication from the machine to the human. It is not apparent what happens
in situations where the communication is from the human to the machine (this may well occur
in collaborative scientific discovery, for example). Symmetry would suggest the existence
of a `Machine Window' and associated requirements of the machine receiving comprehensible messages, but this
is not considered in \citep{michie:ewsl88}\footnote{Although Michie did refer to his approach as, in some sense,
inverting John McCarthy's dictum that ``In order for a program to be capable of learning something it must first be capable of being told it''~\citep{McCa:p:1959}.}. Finally, nothing is proposed by way of
a quantitative or qualitative assessment for one-way intelligibility of the
machine's communication (this is addressed by \citep{mugg:expl}, who propose
a quantitative measure of how beneficial the machine's explanation was to the human).

In this paper, we describe an interaction model between agents that can make
predictions and provide explanations for their predictions. Our focus is not on
developing any specific technique or representation for predictions and
explanations by agents, but
to identify whether the predictions and explanations provided by any one agent
is intelligible to the other. We attempt to do this
by examining the communication between the agents. Specifically:

\begin{enumerate}
\item We propose a communication protocol based on transition systems for
    modelling interactions between agents capable of
    constructing models for data and using these to exchange
    `what' (predict) and `why' (explain) information about data.
    We call such agents {$\PEX$} agents);
\item Based on this protocol we provide definitions
    for One- and Two-Way Intelligibility. 
    When applied to a human interacting with an ML system,
    we identify sufficient conditions
    to ensure that a derivation of One-Way Intelligibility
    using the protocol ensures correctness with the
    derivation of human- or machine-intelligibility using
    a set of `intelligibility axioms'. 
    We also identify conditions under which the
    protocol is complete with respect to recent
    work on viewing explainable AI, or XAI, as a
    property of execution of an extensive argumentation-based
    dialogue model from the literature \cite{madumal}; and
\item We provide case-studies of One- and Two-Way Intelligibility
    from reports in the literature between human and ML systems, in which one
    or both agents employ explanations in symbolic logic, including
    studies from Inductive Logic Programming (ILP). This was the
    original proposal by Michie for Strong and Ultra-Strong machine
    learning. The results there suggest that simply adopting
    logic-based explanations may not be sufficient for One-Way Intelligibility,
    which is consistent with the identification of `harmful' explanations
    in \citep{mugg:expl}.
\end{enumerate}

\section{An Axiomatic Specification of Human-Machine Intelligibility}
\label{sec:axioms}

We motivate the development of a general interaction model between agents capable of
prediction and explanation by looking first at possible criteria for inferring
One-Way Intelligibility of human-machine
interaction. These criteria will then inform the design of a communication protocol
for intelligible interaction in the more general setting.

\begin{myexample}
Consider a research study reported in \citep{covid} on the identification of
Covid-19 patients, based on X-ray images. The automated tool described in
the study uses a hierarchical design in which clinically relevant features
are extracted from X-ray images using state-of-the-art deep neural networks.
Deep neural networks are used to extract features (like ground-glass opacity) from the X-rays, and the system
 also includes a deep network for prediction of possible disease (like pneumonia).
The outputs from the deep networks are used by a symbolic decision-tree learner to
arrive at a prediction about Covid-19. 
Explanations are textual descriptions obtained from the path followed by the decision-tree
when classifying an example.
\begin{figure}[h!]
\begin{center}
    {\scriptsize{\em
    \begin{tabular}{lll} \\
    X-ray Not Covid because: & \hspace*{0.2cm} &  The explanation does not mention \\
    \hspace*{0.1cm} Air-space opacification probability is low;  and & & the right upper lobe air space \\
    \hspace*{0.1cm} Cardiomegaly probability is high; and & & opacification consistent
            with Covid. \\
    \hspace{0.1cm} Emphysema probability is low; and & & \\
    \hspace*{0.1cm} Pneumothorax probability is low; and  \\
    \hspace*{0.1cm} Fibrosis probability is low. \\[8pt]
    {\bf {Machine's explanation}} & & {\bf {Radiologist's feedback}}
    \end{tabular}\\
    }}
    \end{center}
    \caption{
        The machine's explanation for the classification of an X-ray image
        and a senior radiologist's feedback.}        
    \label{fig:covidx}
\end{figure}

\noindent
Results reported in \citep{covid} describe how this neural-symbolic approach compares
to an end-to-end monolithic neural approach (the predictive results of the two are comparable). 
However, our interest here is on the clinical assessment of the explanations produced
by the symbolic model by radiologists.
Figure \ref{fig:covidx} shows an example of a machine's explanation and a clinician's assessment
of that explanation. 
\end{myexample}
\noindent
Later (in Section \ref{sec:reappraise}) we return to this problem and provide a tabulation of assessments
on several ``test'' images. For the present we note that in the study, the human either confirms, refutes, or simply ignores a machine's prediction and explanation. Of these, only the first two
actions could be taken as indicative of intelligibility (although, as we will see later, even
this is not necessarily the case). Here we attempt to capture this
using the following six axioms that are concerned with communication of information:\footnote{
We restrict information to be in the  form of a prediction
accompanied by an explanation. For simplicity, we refer to
both as an explanation in the axioms, and we disentangle these
later in the paper. The constituents of explanations
are left open: the axioms only identify conditions when these
constituents are intelligible to the recipient.}

\noindent
{\bf Human-to-Machine.} Axioms in this
        category are concerned with
        machine-intelligibility
        of the information provided
        by a human to the machine.
       
        \begin{enumerate}
            \item[1.] Machine-Confirmation:
                    If the machine ratifies
                    a human-explanation
                    then 
                    the human's explanation is
                    intelligible to the machine.
             \item[2.] Machine-Refutability:
                    If the machine refutes
                    a human-explanation
                    then 
                    the human's explanation is
                    intelligible to the machine.
          \item[3.] Machine-Performance:
                 If the human-explanation
                    improves machine-performance
                    then the human's
                    explanation is
                    intelligible to the machine.
        \end{enumerate}

\noindent
{\bf Machine-to-Human.} This concerns the human-intelligibility of
        explanations provided by a machine:
        \begin{enumerate}
               \item[4.] Human-Confirmation:
                    If the human ratifies
                    a machine-explanation
                    then 
                    the machine's explanation is
                    intelligible to the human.
             \item[5.] Human-Refutability:
                    If the human refutes
                    a machine-explanation
                    then 
                    the machine's explanation is
                    intelligible to the human.
          \item[6.] Human-Performance:
                 If the machine-explanation
                    improves the human's performance
                    then the machine's explanation is
                    intelligible to the human.
        \end{enumerate}

\noindent
For the example in Figure \ref{fig:covidx}, the condition for the Human-Refutability
axiom will hold. And the machine's explanation will be considered as intelligible for 
the Human. We will look at more examples in detail in Section \ref{sec:reappraise}.
 At this point, the following clarifications may be helpful:
 
 \begin{itemize}
    \item The axioms are not intended to be a complete specification of machine- or
            human-intelligibility. Thus it is possible, for
            example, that none of the conditions for the machine-to-human axioms
            hold, and the machine's explanation may still be human-intelligible;
            The axioms also do not specify what, if anything, should be done if one
            or more of them hold. For example, if a machine's explanation is refuted,
            then what should the machine do about it?  This is normal since the axioms 
            are specifications of intelligibility and not of actions to be done.
    \item Two aspects of the axioms that might escape attention are: (a) Although individually,
        the axioms result in an inference of One-Way Intelligibility, taken together 
        they allow an inference of Two-Way Intelligibility; and (b) 
        The inference of intelligibility will depend on
        the specific human and machine involved in the interaction.
\end{itemize}

 We can at best take the axioms to be a partial specification for
 intelligibility within the context of an interaction model. Below,
 we describe a general model of interaction
 between agents.  We will return in Section \ref{sec:protprop_lxp} to
 the relationship between the intelligibility defined in interaction
 model and the axioms here.
 
 \section{Modelling Interaction between Agents that Predict and Explain}

We describe an interaction model for the more general
setting. For clarity, the description will be semi-formal: a detailed
formal treatment is in Appendix \ref{app:lxp}. 

\subsection{{$\PEX$} Agents and $\PEX$ Automata}

We consider interaction between
agents that have capabilities for learning (induction)
and explanation (justification).\footnote{The capacity for inference (deduction)
is taken for granted.} We will call such agents {$\PEX$} agents
(short for \underline{L}earn-and-\underline{Ex}plain).
Specifically, we assume that the interaction between {$\PEX$} agents 
will be modelled by communicating finite-state automata, which we will
call {$\PEX$} automata. A detailed specification of {$\PEX$} automata is
in Appendix \ref{app:lxp}. For the present, As normal, we will assume
that during run-time, the automaton at any instant is fully specified
by its {\em configuration\/}. Specifically, we will assume that
the configuration of a $\PEX$ automaton associate with agent $a_m$ includes:
a hypothesis $H_m$; and
a dataset $D_m$ consisting of 4-tuples $\{(x_i,y_i,e_i,p_i)\}_{i=1}^N$, where $x_i$ is a
data-instance, $y_i$ is a prediction given $x_i$; $e_i$ is an
explanation for $y_i$; and $p_i$ represents
the provenance for the prediction and the explanation
(that is, details about the origin of $y_i,e_i$ for an $x_i$: a simple
example is the automaton that sent the prediction and explanation). 
Additionally, the $\PEX$ agent $a_m$
has access to the following functions:

\begin{enumerate}
    \item[(a)] ${\texttt{PREDICT}}_m$ that returns the prediction of a 
        data-instance $x$ using its hypothesis;
 \item[(b)] $\texttt{EXPLAIN}_m$ that returns an explanation  for a data-instance $x$ using
        its hypothesis;
\item[(c)] $\texttt{LEARN}_m$ that learns a possibly new hypothesis given its existing
        hypothesis, dataset, and a possibly new data-triple; 
\item[(d)] $\texttt{MATCH}_m$ which is true if a pair of predictions $y, y^\prime$ match; and
\item[(e)] $\texttt{AGREE}_m$ that is true if a pair of explanations $e, e^\prime$
agree with each other.
\end{enumerate}

\noindent
We use the term $\PEX$-functions for the functions (a)--(e) above.
$\PEX$ agents are correctly $\PEX$-function, $\PEX$ automata pairs.
In the rest of the paper, it is understood
that $\PEX$-functions are agent-specific, and we will drop the subscript on
the functions unless required for emphasis.  Also
the ${\PEX}$ automaton defined in Appendix \ref{app:lxp} use the
agent-specific $\PEX$ functions
to define guarded transition relations, and we will use the term
``$\PEX$ automaton'' interchangeably with the corresponding
$\PEX$ agent, and the agent-specific ${\PEX}$-functions will be
associated with the corresponding automaton.
We will also assume a special agent $\Delta$, called the
{\em oracle\/}. $\Delta$ is a non-{$\PEX$} agent, but
it will be convenient to model its interaction with other {$\PEX$} automata
using the same communication protocol used for ``normal'' {$\PEX$} automata.

We do not commit at this point to any specific form taken by the predictions or the
explanations. We also leave open what is meant by a pair of predictions matching or
a pair of explanations agreeing: these will depend on the actual form taken by
the predictions and explanations. For example, if {\tt{PREDICT}} returns a numeric value,
then a pair of predictions could be assumed to match is they are within
some tolerance.\footnote{However, then the definition of ${\mathtt{MATCH}}$ may
not satisfy some intuitive properties of equality (like transitivity).} We assume that $\Match$ function is commutative (that is if $\Match(a,b)=true$, then 
$\Match(b,a)=true$). 

The {\tt{EXPLAIN}} and {\tt{AGREE}} functions may not be straightforward. However, for
some kinds of agents, like those that provide logic-based explanations
it is possible to identify ${\mathtt{EXPLAIN}}$  with some known descriptors, like proofs
and ${\mathtt{AGREE}}$ can be formulated in terms of well-understood
        logical operations (see Example \ref{ex:logicpex} below).
        For explanations in a less formal setting, like natural language, 
        it is likely that obtaining a definition of
        ${\mathtt{AGREE}}$ may require additional effort, and
        may require models constructed from data to decide agreement.
        
\begin{myexample}[Logic-based $\PEX$ functions]
\label{ex:logicpex}
Let $a_m$ and $a_n$ be agents that use a logic-based representation,
for hypotheses and explanations, as is the case in
Inductive Logic Programming, or ILP \citep{Mug-DeR:j:94}.
Let $H_m$ be the current hypothesis of $a_m$ and $H_n$ be the hypothesis for $a_n$.
Let predictions of a data-instance $x$ by $H_{m,n}$ be done by clauses of
the form $predict(X,C) \leftarrow {Body}$, to be read as ``The prediction of any instance
X is C if the conditions in $Body$ are true''. Then possible
${\PEX}$ functions for $a_m$ (and similarly for $a_n$) are;

\begin{itemize}
\item[(a)]  $y = {\mathtt{PREDICT}}_m(x,H_m)$ $\equiv$ $(H_m \vdash predict(x,y))$
    (where $\vdash$ is a derivability relation);
\item[(b)] ${\mathtt{LEARN}}_m$ constructs hypotheses using techniques developed
    in ILP;
\item[(c)] $e = {\mathtt{EXPLAIN}}_m((x,y),H_m)$ is the clause in $H_m$
    used to derive $predict(x,y)$;
\item[(d)] If $y_m$ is a prediction by $a_m$  and
        $y_n$ is a prediction by $a_n$ then
        ${\mathtt{MATCH}}_m(y_m,y_n)$ $:=$ $(y_m = y_n)$;
\item[(e)] if $e_m$ is a (clausal) explanation from $a_m$ and
        $e_n$ is a (clausal) explanation from $a_n$ then
        ${\mathtt{AGREE}}_m(e_m,e_n)$ $:=$ $(e_m =_\theta e_n)$ 
        (where $=_\theta$ denotes an equivalence relation
        based on the $\theta$-subsumption as defined
        in \citep{plotkin:thesis}).
\end{itemize}
(These definitions are illustrative, and not the only ones possible with logic-based
agents.)
\end{myexample}

We will also require that
only messages sent by $\Delta$ can contain $\blacktriangle$ as an explanation
and that the following restriction holds on the {$\PEX$} functions.

\begin{myremark}
\label{rem:oracle}
For any agent $m \neq \Delta$ and $D_m$, if $H_m = {\mathtt{LEARN}}_m(\cdot,D_m)$
and $(x,y,\blacktriangle,\Delta)\in D_m$,  then ${\mathtt{MATCH}}_m({\mathtt{PREDICT}}_m(x,H_m),y) = true$.
\end{myremark}
Informally, this assumes the predictions by the oracle are always correct, 
and therefore all non-oracular agents have to ensure their predictions are consistent with the predictions
they have received from the oracle.

\subsection{Communication between Automata}

We focus on sequences of pairwise interactions between {$\PEX$} automata.
Each sequence is called a {\em session\/}. Let us informally
call a pair of automata {\em compatible\/} within a session if there is
a consensus between them on the prediction- and explanation-pairs that are the same.\footnote{That is, for automata $a_m$ and $a_n$ in a session, if 
$y_m ~=~ \tt{PREDICT}_m(x)$ and $e_m ~=~\tt{EXPLAIN}_m(x)$ is a prediction-explanation pair for $x$ by $a_m$ and $(y_n,e_n)$ is
a prediction-explanation for $x$
by $a_n$, then ${MATCH}_m(y_m,y_n) ~=~ \mathtt{MATCH}_n(y_n,y_m)$ and 
$\mathtt{AGREE}_m(e_m,e_n) ~=~ \mathtt{AGREE}_n(e_n,e_m)$. More details can
be found in Appendix \ref{app:lxp}.} 

Within a session automata send each other tagged messages. The messages consist of
the sender, a tag, data-instance, prediction, and explanation.
For the present, we focus on the message-tags. Suppose $a_m$ 
$a_n$ are compatible automata in a session.  If
automaton $a_n$ receives a
message from $a_m$ with prediction $y_m$ and explanation $e_m$, then, $a_n$
sends a message to $a_m$ either terminating the session (message
tag ${\mathit{TERM}}$), or a message that contains one of 4 tags:
$\Ratify$, $\Revise$, $\Refute$ or $\Reject$. 
Formally, each automaton employs a guarded transition system,
 in which mutually-exclusive guards are used to identify which message-tag to choose. 
We refer the reader to Appendix ~\ref{app:lxp} for further details.

\subsection{{\tt{PXP}}: A Communication Protocol for {$\PEX$} Automata}
\label{sec:prot}
We adopt a protocol for messages sent by $\PEX$ automata.
We introduce the protocol for communication using single instance:
        This restriction can be easily relaxed by allowing messages that communicate about 
        multiple instances, their predictions and explanations.\footnote{The usual representation for this would be using $N$-dimensional vectors, where $N$ is the number of instances.}
Let ${\cal A}$ denote the set of $\PEX$ automata; ${\cal X}$ the set of instances;
${\cal Y}$ the set of predictions; and ${\cal E}$ the set of explanations.
Messages sent by a $\PEX$ automaton are of the form $+(m,(t,(x,y,e)))$, and messages received
are of the form $-(m,(t,(x,y,e)))$;  
where  $m \in {\cal A}\cup \{\Delta\}$, 
 $t \in \{\mathit{Init},\mathit{Ratify},\mathit{Refute},\mathit{Revise},\mathit{Reject},\mathit{Term}\}$, 
 $x \in {\cal X}$, $y \in {\cal Y} \cup \{`?'\}$, $e \in {\cal E} \cup \{`?',\blacktriangle\}$. 
Here  `?' is to be read
as ``not known''; and The explanation $\blacktriangle$ is to be read as ``oracular statement''.

\begin{figure}[ht]
    \centering
    \includegraphics[scale=0.4]{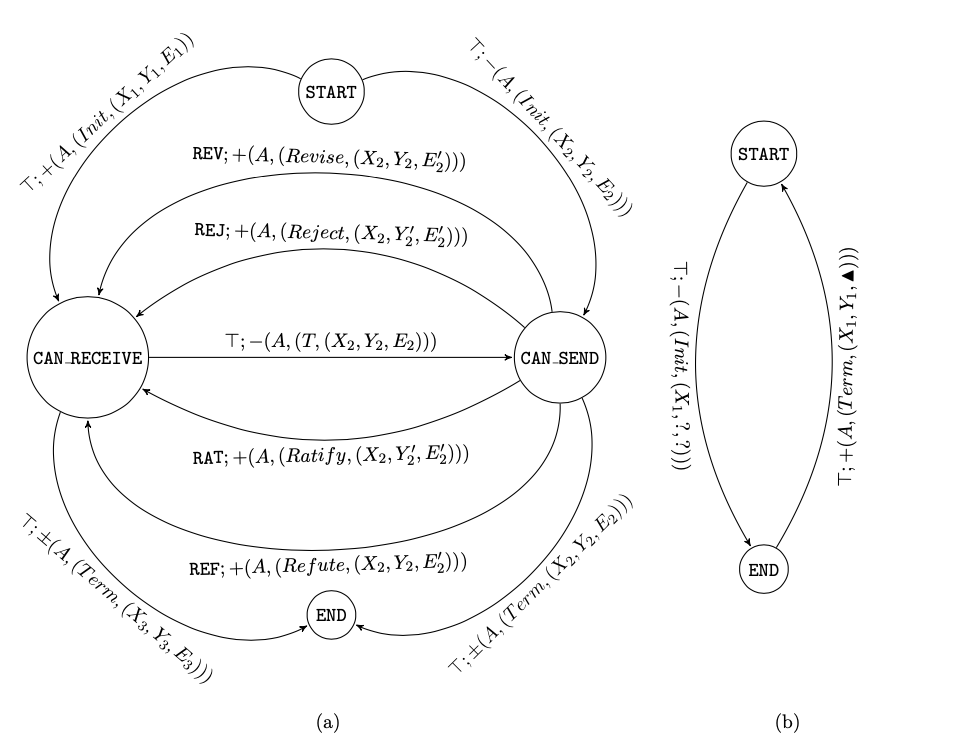}
\caption{Messages sent ('+") and received ('-') in ${\PXP}$ by: (a) automata for agents other
than the oracle; and (b) the oracle. Here $\top$ stands for a guard condition
that is trivially true. {\tt{RAT}, \tt{REF}, \tt{REV}} and {\tt{REJ}} represent the
guard conditions used by the guarded transition system, which are described
below.}
\label{fig:fsm}
\end{figure}

Figure \ref{fig:fsm}(a) shows
the messages sent and received by
an automaton for an agent (other than $\Delta$), and Figure \ref{fig:fsm}(b) shows
the corresponding messages sent and received by $\Delta$.
Informally, the figure tells us that every session between
a pair of {$\PEX$} automata has to explicitly initiated and terminated.
The session can be terminated by either automaton, and an automaton can only 
initiate a new session after terminating an existing session. 
Communication is therefore like a plain-old-telephone-system (``POTS'').
    
In Figure \ref{fig:fsm}, we want agents to send `?' for predictions and/or explanations only when they
initiate the sessions. To enforce this restriction, we will make the following assumptions about
the $\PEX$ functions: for any agent $m$, $\mathtt{PREDICT}_m$ and $\mathtt{EXPLAIN}_m$ will not return `?'  
and $\mathtt{MATCH}_m$ and $\mathtt{AGREE}_m$ will return false if any one of the arguments is  `?'.  

\noindent
To specify the $\PXP$ protocol fully we need to define the transition system.
This is done in Appendix \ref{app:lxp}.

\begin{myremark}[Need for $\PEX$ Functions and Compatability]
The execution of $\PXP$ requires definition
of the ${\PEX}$ functions. As is evident
from Definition \ref{def:trans} in Appendix \ref{app:lxp}, transitions have checks and updates that involve significant local computation.
The protocol has been substantially simplified
by assuming that the agents involved
in the interaction are compatible. 
The compatibility is about a shared meaning of
when two predictions are similar, and when two explanations agree. Without
this common understanding, the communication patterns become complex in general. 
\end{myremark}

\begin{myremark}[Synchronous Communication]
$\PXP$ is a synchronous communication
protocol, and interaction between a pair of agents has to be terminated before
commencing a new one. More elaborate protocols allowing concurrency may
be possible, at the expense of greater complexity in managing the global
configuration of the system. It may require provenance
information to be more detailed (like inclusion of session identifiers and
 session indices, along with the sender's identifier).
\end{myremark}

\begin{myremark}[Noise-free Channels]
The protocol does not account for noise in the channel, delays, or cost
of communication between agents. Implementations will need to account for
all of these aspects.
\end{myremark}

\begin{myremark}[Oracular Communication]
A question that arises is this: if the true label of data-instances
can be obtained directly from the oracle, why doesn't a $\PEX$ agent
communicated just with $\Delta$?
One aspect we have not considered in
developing the communication protocol is the cost of
communication. Collaboration between
${\PEX}$ agents will be worthwhile if:
(a) communication
to and from $\Delta$ is significantly more expensive
 (in time or money or both) than communication between ${\PEX}$ agents; and
 (b) the $\PEX$ functions of any one agent can use predictions and
 explanations from other agents effectively. Of these, (a) is likely
 to be the case if the oracle is intended to model the acquisition of
real-world data by manipulation and experimentation. The extent
to which (b) holds will depend on whether agents are able to
establish some common knowledge, and fulfil the requirements of
compatibility. 
\end{myremark}

Finally, it is common for protocols to have a preliminary (hand-shaking)
phase where some prior information is exchanged. We have not described this aspect
in the paper, but it is the phase where ${\PEX}$ agents can establish
some `common ground' needed for ensuring compatibility. We now turn to the
principal motivation for introducing the protocol, namely as a syntactic basis  for
identifying intelligible interaction.

\subsection{Intelligibility from Interaction}
\label{sec:intell}
We now have the pieces to define intelligibility as a
syntactic property that follows from the execution of the $\PXPk$ protocol.
We first give an informal description by what we propose for one-way intelligibility.
For the messages sent from $a_m$ $a_n$ to be intelligible to $a_n$, we 
want the corresponding response by $a_n$ at each step to be: (a) one of ratification, refutation
or revision; and (b) not a rejection. Additionally, we also require: (c) $a_n$ should not
refute $a_m$ at every step.\footnote{This will restrict what we consider intelligible.
For example, suppose $a_m$ provides a mathematical proof as explanation for its prediction
For a data instance $x$. $a_n$ refutes the explanation and terminates the session. With
restriction (c), $a_m$'s message is not intelligible to $a_n$. Purely by looking
at the message tag (here $\Refute$) we are unable to
distinguish if this is because $a_n$ has understood, but disagrees with steps in the proof,
or simply cannot understand the proof.} Inclusion of (c) allows a simpler version of (a) in the
definition of one-way intelligibility.

\begin{mydefinition}[One-Way Intelligibility]
\label{def:1wayintell}
Let $S$ be a session between compatible
agents $a_m$ and $a_n$ using $\PXP$. Let $T_{mn}$ and $T_{nm}$ be the
sequences of message-tags sent in a session $S$ from $m$ to $n$ and from $n$ to $m$. 
We will say $S$ exhibits One-Way Intelligibility for $m$ iff
(a) $T_{mn}$ contains at least one element in
    $\{\Ratify,\Revise\}$; and
(b) there is no $\Reject$ in the sequence $T_{mn}$.  
Similarly for  One-Way Intelligibility for $n$.
\end{mydefinition}

\noindent
(This does not mean $\Refute$ is unimportant. In $\PXP$, the receipt of a $\Refute$
message-tag is one way to initiate a revision, which may then result in a response
with a $\Revise$ tag.)

\begin{mydefinition}[Two-Way Intelligibility]
\label{def:2wayintell}
Let $S$ be a session between compatible
agents $a_m$ and $a_n$ using $\PXP$. Let $T_{mn}$ and $T_{nm}$ be the sequences of message-tags sent in a session $S$ from $m$ to $n$ and from $n$ to $m$.
We will say $S$ exhibits Two-Way Intelligibility for $m,n$  iff $S$ is One-Way Intelligible for $m$ and $S$
is One-Way Intelligible for $n$ using $\PXPk$.
\end{mydefinition}

\begin{myremark}[Strong-Intelligibility Ultra-Strong Intelligibility]
Insipired by the properties of Strong and Ultra-Strong ML in
\cite{michie:ewsl88},  we suggest the following:
(a) If every interaction between a human and an ML system is
One-Way Intelligible for the human then we will say the  ML system is strongly intelligible for
the human; and 
(b) If an ML system is strongly intelligible for the human, and there exists at
least one interaction with a  $\Revise$ message-tag in the
message sequence sent from human to the ML system, then we will say that the ML system exhibits
ultra-strong intelligibility for the human.
\end{myremark}

\noindent
We note that this does not restrict the ML system to provide symbolic explanations,
as is required in \cite{michie:ewsl88}. In Sec.~\ref{sec:reappraise} we look specifically at
some real use-cases from the literature where one or both of human and machine
provide explanations in symbolic logic. We first illustrate some application of
the intelligibility definitions to some hypothetical human-machine
interaction.

\subsubsection{Human-Interaction With Some Hypothetical ML Systems}
We consider sessions about a data-instance $x$
between a human $h$ and a ML system $m$. We assume the
following:
(a) The machine initiates the interaction by sending the human a
    prediction and an explanation;
(b) The human sends at least one message back to the machine before termination; and
(b) The human can terminate the session after one or more messages have been sent or received.
None  of these conditions are necessary for $\PXP$, and are adopted here for simplicity.
The categorisation below is purely expository, and not intended as any kind
of classification of ML engines.

\noindent
{\bf{Lucky Machine.}} Suppose the machine sends a prediction for $x$ that agrees
by chance with the human's, but the explanation is gibberish. The human refutes the
explanation with a correct one, and terminates the session. The message-tag sequence is then
$\langle \Init_m, \Refute_h,$ $ \Term_h \rangle$.
The agent-specific tag-sequences are
$T_{hm}$ = $\langle \Refute_h, \Term_h \rangle$ and $T_{mh}$ = $\langle \Init_m \rangle$
This is not One-Way Intelligible for human or machine, and therefore
not Two-Way Intelligible.
The machine can continue to be lucky, and revise its hypothesis in a way that continues to agree with the
human, but explanations continue to be nonsense. The interaction tag-sequence would then
extend to be of the form
$\langle \Init_m,\Refute_h,\Revise_m,$ $\Refute_h,$ $\Revise_m,\cdots,\Term_h \rangle$.
This is One-Way Intelligibility for the machine but not for the human. 
The session is therefore not Two-Way Intelligible.

\noindent
{\bf{Obdurate Machine.}} We now consider the a variant of the case above, in which the machine does not
(or cannot) revise its model. The sequence
$\langle \Init_m, \Refute_h,$ $ \Refute_m,$ $\Refute_h,$ $\Refute_m,$ $\cdots,\Term_h \rangle$.
This is neither One-Way Intelligible for the human nor the the machine. Obviously,
it is also not Two-Way Intelligible.

\noindent
{\bf{Compliant Machine.}} Suppose the machine is willing to revise its hypothesis to comply with
the human's prediction and explanation. An example tag sequence is
$\langle \Init_m, \Refute_h,$ $ \Revise_m, \Ratify_h, \Term_h \rangle$. This is One-Way
Intelligible for both human and machine, and therefore the session is
Two-Way Intelligible. Extended versions like
$\langle \Init_m,\Refute_h,$ $\Revise_m, $ $\Refute_h,$ $\Revise_m,\cdots,$ $\Ratify_h, \Term_h \rangle$
will similarly be Two-Way Intelligible, and is an example of the human
``teaching'' the machine.

\noindent
{\bf{Helpful Machine.}} Suppose the machine's hypothesis results in the human
revising his or her hypothesis (usually to improve performance). 
An example tag sequence is
$\langle \Init_m, \Revise_h,$ $ \Ratify_m, \Term_h \rangle$. This is One-Way
Intelligible for both human and machine, and therefore the session is
Two-Way Intelligible. Extended versions like
$\langle \Init_m,\Revise_h,$ $\Refute_m, $ $\Revise_h,$ $\Ratify_m, \Term_h \rangle$
will similarly be Two-Way Intelligible, and constitute an example of the machine
teaching the human.

\noindent
{\bf {Incomprehensible Machine.}} The machine sends a message which
the human simply cannot understand (both prediction and explanation). With
a reasonable definitions for $\Match$ and $\Agree$, the tag-sequence that results is
$\langle \Init_m, \Reject_h,$ $ \Term_h \rangle$. This is neither One-Way Intelligible
for human nor machine. 

\section{Properties of $\PXP$}
\label{sec:protprop_lxp}
\subsection{Termination}
We construct an abstraction of the set of transitions of $\PXP$
in the form of a `message-graph' (for details see Figure \ref{fig:mgraph} in 
Appendix \ref{app:lxp}). Due to cycles and self-loops in the message graph, the communication can 
become unbounded. If we consider only compatible agents then there will be no
cycles in the message graph (see Proposition \ref{prop:illegal}  in Appendix \ref{app:lxp}). The length of 
communications might still be unbounded due to  the presence of self-loops in the 
message graph. We modify the protocol  such that
each of the self-loop can occur at most $k$ times. We call this modified protocol $\PXPk$. 
It is straightforward to show that communication between compatible agents using ${\PXPk}$ is bounded. 

\begin{myproposition}[Bounded Communication]
Let Figure \ref{fig:mgraph}(c) represents the message graph of a collaborative session 
using the $\PXPk$ protocol. Then any communication in the session has
bounded length.
\end{myproposition}
Proof for this proposition is in Sec. \ref{app:proofterm} in Appendix \ref{app:lxp}.

\subsection{Correctness wrt the Intelligibility Axioms}
\label{sec:correct_axiom}

We identify conditions under which 
we can establish correctness of the {$\PXP(k)$} protocol wrt to
the Intelligibility Axioms in Sec.~\ref{sec:axioms}.
Here, by correctness we mean that when One-Way Intelligibility follows from Def. \ref{def:1wayintell} using $\PXP$, we would like to infer
intelligibility from the axioms.

\begin{mydefinition}[Actuation Constraints]
\label{def:act}
Let $S$ be an $\PXP$ session between human and
machine for a data instance $x$. Let $y_h$ and $y_m$ denote the prediction
by the human and machine respectively for $x$. Let $e_h$ and $e_m$ denote 
the human's and machine's explanation respectively.
Let $y_h^\prime$ and $y_m^\prime$ denote the human's and machine's prediction 
after hypothesis revision using $\Learn_h$ and $\Learn_m$ respectively.
Let $e_h^\prime$ and $e_m^\prime$ denote the human's and machine's prediction 
after hypothesis revision using $\Learn_h$ and $\Learn_m$ respectively.
We define the following Actuation Constraints on the
$\PEX$ functions:

\begin{description}
    \item[AC1h] If $~\Match_h(y_h,y_m)~\wedge~\Agree_h(e_h,e_m) ~=~ true$ then
        the human ratifies the machine's explanation $e_m$.
     \item[AC1m] If $~\Match_m(y_m,y_h)~\wedge~\Agree_m(e_m,e_h) ~=~ true$ then
        the machine ratifies human's explanation $e_h$.
    \item[AC2h] If $~\Match_h(y_h,y_m) ~\wedge~ \Agree_h(e_h,e_m) ~~~=~~false$ and 
    $~\Match_h(y_h\prime,y_m)~\wedge~$ $\Agree_h(e_h\prime,e_m)~ = ~true$ then
    there is improvement in the human's performance.
    \item[AC2m] If $~\Match_m(y_m,y_h)~\wedge~\Agree_m(e_m,e_h) ~=~ false$ and 
$~\Match_m(y_m\prime,y_h)~\wedge~$ $\Agree_m(e_m\prime,e_h) ~=~ true$ then
    there is improvement in the machine's performance.
\end{description}   
\end{mydefinition}

\begin{myproposition}
Let $S$ be a session between a human and a machine using $\PXP$ and the 
$\PEX$ functions satisfy the Actuation Constraints. 
If $S$ exhibits One-Way Intelligibility for $h$ then
the antecedent of one of the Machine-To-Human axioms is true. Similarly if $S$ exhibits
One-Way Intelligibility for $m$ then the antecedent
of one of the Human-To-Machine axioms is true.
\end{myproposition}
\begin{proof}
Let us assume $S$ exhibits One-Way Intelligibility for 
the human $h$.  Let $T$ be the sequence of message-tags sent in 
$S$ from $h$ to $m$. By the definition of One-Way 
Intelligibility, $T$ has either $\Ratify$, or  $\Revise$. Let 
$y_h, y_m, e_h, e_m$ be the prediction of human and machine and 
explanation of human and machine respectively before sending this 
message tag. Let $y_h^\prime, y_m^\prime, e_h^\prime, e_m^\prime$ 
be the prediction of human and machine and explanation of human 
and machine respectively after revising the hypothesis.

Suppose $T$ has $\Ratify$. 
$m$ sent $\Ratify$ only if $\Match_h(y_h,y_m)$ is true and 
$\Agree_h(e_h,e_m)$ is true. Hence $m$'s explanation is 
ratified by $h$ by the Actuation Constraint AC1h; and the/
the antecedent of the Human Confirmation axiom is true in the
Machine-To-Human axioms. 

Suppose $T$ has $\Revise$. 
It follows from Definition~\ref{def:guards} in Appendix \ref{app:lxp} that
$h$ sends $\Revise$ iff
$\Match_h(y_h,y_m)~\wedge~\Agree_h(e_h,e_m) ~=~ false$ and
$\Match(y_h^\prime,y_m)~\wedge~$ $\Agree_h(e_h^\prime,e_m) ~=~ true$.
From Actuation Constraint AC2h the human's performance improves,
and the antecedent of the Human Performance is true in the
Machine-To-Human axioms.

\noindent
Similarly for the Machine Confirmation and Machine Performance axioms
in the Human-to-Machine axioms.
\qed
\end{proof}

Thus, the Actuation Constraints are sufficient to establish correctness of
$\PXP$ wrt the axioms in Sec.~\ref{sec:axioms}.

\subsection{Agreement with the Oracle}

The oracle $\Delta$ has been used as the source of
infallible  information  about the label for any data instance $x$.
$\Delta$ terminates the session with the message $+(m,(Term,(x,y,$ $\blacktriangle)))$, where
$y$ is the oracle's prediction for the label of $x$ and $\blacktriangle$ denotes that
the explanation is an oracular statement.
In our interaction model, $\Delta$ never initiates a session; and never sends any
message-tags other than $Term$.


\begin{myproposition}
Let $m,n$ be compatible ${\PEX}$ agents such that their $\Match$ functions are transitive ($\Match(a,b) \land \Match(b,c) =true $ then $ \Match(a,c)$ is true).
If either $m$ or $n$ has an oracular prediction for $x$ and the session ends with $m, n$
reaching a consensus on prediction and explanation, then both agents will agree with
the oracle's prediction for $x$. 
\end{myproposition}
\begin{proof}
Let us assume that $m$ has an oracular
prediction $y$ for $x$ which  means $(x,y,\blacktriangle,\Delta) \in D_m$ and
$\Match_m(y,\Predict(x,H_m))=true$ for any
hypothesis $H_m$ constructed by $m$ after receiving the oracular
prediction for $x$ (by Remark \ref{rem:oracle}). 
Let the session $S_{mn}$ for $x$ ends with $m, n$ reaching a 
consensus on prediction and explanation with $H_m$ and $H_n$ be
the hypothesis for $m$ and $n$ respectively. It means 
$\Match_m(\Predict(x,H_m),\Predict(x,H_n))=true$. Since $\Match_m$ is transitive, we
conclude that $\Match_m(y,\Predict(x,H_n))=true$. 
Since $m$ and $n$ are compatible, $\Match_n(\Predict(x,H_n),y)$ is true. So both the
agents will agree with the oracle's prediction for $x$.
\end{proof}

\begin{myremark}
 It is sufficient for only one of $m$ or $n$ to
communicate to $\Delta$ about $x$. Extending to set of instances $X = \{x_1,x_2,\ldots,x_k\}$,
the cost of communicating to the oracle can be reduced for both $m$ and $n$ by restricting
oracle-communication for each agent to partitions $X_m$ and $X_n$ respectively.
\end{myremark}

\subsection{Completeness wrt a Dialogue Model}
\label{sec:correct_dial}
In \citep{madumal} the authors propose a graphical representation of interactions
in a dialogue between a questioner Q and an explainer E. Although in principle,
Q and E could be either human or machine, the interactions are most easily understood
in the context of a human questioner and an machine-learning system as explainer.  
The interaction graph is in Figure \ref{fig:madprot}.

\begin{figure}[ht]
    \centering
    \includegraphics[scale=0.45]{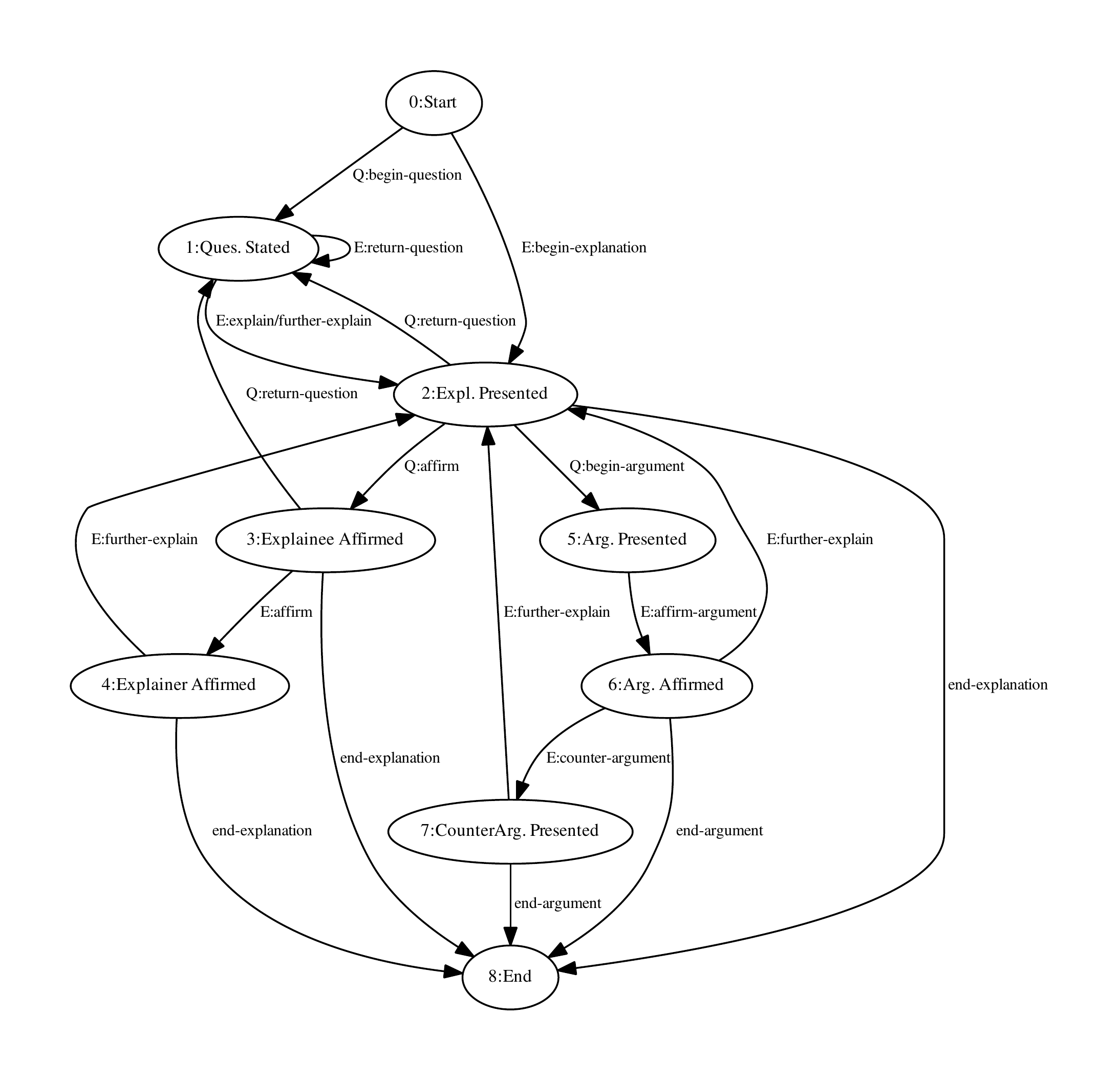}
    \caption{The interactions proposed in \protect{\citep{madumal}}.
    The node-labels represent
        states and edge-labels representing actions (the prefix on the edge-label
        denotes the originator of the action). The node-numbering is ours.}
    \label{fig:madprot}
\end{figure}

Here, we will compare the graph in Fig.~\ref{fig:madprot} \footnote{we will call ${\mathtt{MMSV}}$--based on the
proposers, was (manually) constructed by
examining interaction transcripts, and the result clearly reflects
constructs identifiable in dialogues between humans.} with the $\PXPk$ protocol. Let us denote ${\mathtt{MMSV}}$ without the edge $(1,1)$ as
${\mathtt{MMSV}}^-$. 
An example of conversations of length up to 5
are shown in Table \ref{tab:madcompare}. The corresponding $\PXPk$ transitions are in Table \ref{tab:madcompare}. 
We show that all bounded length ``conversations" allowed in  ${\mathtt{MMSV}}^-$ have corresponding finite length interaction using $\PXPk$. Additionally, we show that all interactions using
$\PXPk$ are Two-Way intelligible. That is:

\begin{myproposition}
Every path of length $l$ in the graph for ${\mathtt{MMSV}^-}$
corresponds to a path of at most $l$ using $\mathtt{PXP(l)}$.
\end{myproposition}
The proof for this proposition is in Appendix \ref{app:madlxp}. We note
also that not all paths of length $l$ in the graph for ${\mathtt{MMSV}^-}$ map
to One-Way or Two-Way Intelligible interactions with $\mathtt{PXP(l)}$. 

\section{Case Studies: Agents Providing Logic-Based Explanations}
\label{sec:reappraise}

In this section, we re-cast some reports from the
literature involving humans and ML systems
as interactions demonstrating
One-Way or Two-Way Intelligibility. We focus specifically on
cases where the human or ML system employ logic-based explanation
(ML systems developed in the category of Inductive Logic Programming,
or ILP, are in this category).
The purpose of this exercise is three-fold. First, it demonstrates
how the interaction model we propose can be applied to
characterise the interaction between humans and ML systems.
Secondly, it highlights the difference between human-machine interactions that are One- and Two-Way Intelligible. Thirdly, it shows simply using ML systems capable of symbolic
explanation does not automatically imply Two-way intelligibility.

\subsection{Examples of One-Way Intelligibility}
\label{sec:onewayex}

\noindent
{\bf {Case Study 1.}}
We start by returning to the classification of Covid X-rays in
Section \ref{sec:axioms}. The radiologist's assessment of the ML model
by the radiologist on 30 ``test'' images is shown in Table \ref{tab:covidxredux}.
For each image, the machine provides a prediction along with a symbolic explanation.
We also tabulate a summary of the textual feedback provided by the 
radiologist. For this evaluation, we ignore the phase of model-construction
by the ML system, and focus only on the interaction of the radiologist
and the ML system on the test data. In obtaining a mapping
of this interaction to ${\PXPk}$ tags in
the last column, we have assumed the following:
(i) If the radiologist marks the prediction as correct and the
    explanation as sufficient, then this
    maps  to sending a message with a ${Ratify}_h$;
(ii) If the radiologist marks the  explanation as insufficient
    (irrespective of the assessment of the prediction)
    then this  maps to sending a message with a ${Refute}_h$ tag;
(iii) If the radiologist marks the explanation as incorrect but the
prediction as correct or unsure and 
provides clarifications as to why the explanation incorrect, then this maps to sending
    a message with a ${Refute}_h$ tag;
(iv) If the radiologist marks the explanation and prediction as incorrect and
    provides no clarifications, then this maps to sending a
    message with a ${Reject}_h$ tag.
\begin{table}
\begin{center}
            {\scriptsize{
                \begin{tabular}{|l|c|c|c|c|c|} \hline
                \multicolumn{4}{|c|}{Radiologist's Opinion about Model's} & Radiologist's & $\PXPk$'s\\ \cline{1-4}
                           & \multicolumn{3}{c|}{Prediction} & Feedback & Tags \\ \cline{2-4}
                Explanation & Correct   & Wrong & Unsure    &   &\\ \hline
                Sufficient & 17 & 0  & 0 & Clarifies: 17 & ${\Ratify}_h$: 17\\
                Incomplete & 1 & 3 & 1 &   Clarifies: 5 & ${\Refute}_h$: 5\\
                Incorrect & 3 & 2 & 3 &    Clarifies: 6 & ${\Refute}_h$: 6, ${\Reject}_h$:2\\\hline
            \end{tabular}
            }}
        \end{center}
    \caption{Radiologist's Assessment of machine predictions and explanations for Covid. Also
    shown is a summary of the radiologist's actions and a mapping of these
    an ${\PXPk}$ message. 
    }
    \label{tab:covidxredux}
\end{table}

In $\PXPk$ terms, the radiologist refutes
the explanations in 11/30 instances, rejects 2/30 instances and ratifies 
them in  remaining 17 cases.
There is no provision in
the system in \citep{covid} for any further interaction (for
example, to revise the machine's model). Thus the interaction
between radiologist and the ML-system in \citep{covid} is
characterised by one of 3 tag-sequences:
(a) $\langle {\Init}_m,{\Ratify}_h,{\Term}_{h/m} \rangle$ (17 cases); 
(b) $\langle \Init_m,\Refute_h, \Term_{h/m}\rangle$ (11 cases); or 
(c) $\langle {\Init}_M,{\Reject}_H,{\Term}_{h/m} \rangle$ (2 cases). Here $h/m$ is used
to denote either human or machine. By
Definition \ref{def:1wayintell} only
the 17 instances in (a) are One-Way Intelligible for the human.

\vspace*{0.3cm}
\noindent
{\bf {Case Study 2.}}
As a second illustration of One-Way Intelligibility, we consider 
information provided by a human as logical statements provided to an ML system in the form
of relevant domain-knowledge. We re-examine
recent results reported in \citep{dash:botgnn,dash:drm}
in the area of drug-design. This  is well-known to be a time-consuming and
expensive process, requiring the involvement of significant amounts of human chemical
expertise. ML-based models have been used to assist in this by constructing
models by using chemical structure to predict
molecular activity (like binding affinity, toxicity,
solubility and so on).
The reported experiments consider the construction of predictive
models using relations which are typically used to construct human-understandable
explanations
(like known chemical ring structures, functional groups, motifs and the like).
The data provided were in the form of logical statements constituting ``most-specific explanations''
for molecules using the domain-knowledge.\footnote{These are logical clauses constructed
from human-supplied domain-knowledge using a technique developed in Inductive Logic Programming \citep{mugg:progol}).} 
The authors tabulated predictive performance for
two kinds of deep neural networks (a multi-layer perceptron, or MLP, and a graph neural network, or GNN). Table \ref{tab:tox1} lists
the number of datasets on which performance improvements are observed,
with the inclusion of the most-specific explanations.
In \citep{dash:botgnn}, the authors recorded a limitation of their
approach: the models constructed by the deep networks are
not understandable by a chemist.

 \begin{table}[t]
    \begin{center}
        \begin{tabular}{|l|ccc|} \hline
               & \multicolumn{3}{|c|}{Comparative Performance} \\
    DNN    & \multicolumn{3}{|c|}{(with domain-knowledge)} \\ \cline{2-4} 
        Type   & Better & Same & Worse \\ \hline
        MLP    &   71   &  0   &  2       \\
        GNN    &   63   &  9   &  1     \\ \hline 
        \end{tabular}
    \end{center}
    \caption{The use of human-supplied domain-knowledge by two kinds
        of deep neural networks (DNNs): mult-layer perceptrons (MLPs) and graph-neural
        networks (GNNs). Estimates of performance are obtained on 73 different
        datasets. Here, ``Better'' (respectively,
        ``Same'' and ``Worse'') means
        the use of domain-knowledge results in an
        improvement in performance (respectively,
        no change, and worse performance).}
    \label{tab:tox1}
\end{table}

So far, we have only described $\PXPk$ as dealing
with a single instance, it would appear inappropriate to
use describe the interaction between human and machine involving
a dataset with many instances. However, there is no
reason why this should be the case.
The message from an agent can be about a set of
instances, usually represented by a vector ${\mathbf x}$.
In this case, we will take the human as initiating a message with
predictions and explanations $(\mathbf{y},\mathbf{e_{y| x, B}})$ about $\mathbf{x}$.\footnote{Each
entity is a $N$-dimensitional vector, where $N$ is the number
of instances in ${\mathbf x}$. For any entry $x$ in $\mathbf{x}$,
there are corresponding entries $y$ in $\mathbf{y}$ and $e_{y|x,B}$ in $\mathbf{e}_{\mathbf{y}|\mathbf{x},B}$,
Here $y$ is the (human)'s prediction of $x$ and $e_{y|x,B}$ is the most-specific
explanation for the prediction $y$ given the human's background knowledge. $y|x,B$
should be read as $y$ given $x$ and $B$.} For the purpose of this exercise, we will assume the ML engine
already has a model constructed without the use of background knowledge
for the dataset $\{(\mathbf{x}_i,\mathbf{y}_i)\}_1^N$ that
matches predictions in $\mathbf{y}$. One example is the trivial model that simply consists of
a lookup-table of the data and predictions provided. 

On receipt of the message from the human, the prediction of the machine's  current  model
for $\mathbf{x}$ will correctly be $\mathbf{y}$. That is,
$\Match_m(\cdot,\cdot)$ will be true. The question is only whether the human's explanation
$e_{{\mathbf y}|{\mathbf x},B}$ will match the machine's explanation, which is denoted by
$e_{{\mathbf y}|{\mathbf x},\emptyset}$. If $\Agree_m(e_{{\mathbf y}|{\mathbf x},\emptyset},e_{{\mathbf y}|{\mathbf x},B}) = false$, then as per the
definition of $\PXPk$ the machine will send either a $\Revise$ or $\Refute$ tag.
For any ${\mathbf x}$ if the Actuation Constraints are true and
the ML engine can alter its model to improve performance after
receiving the human's message then the machine will send a $\Revise$ tag,
otherwise the machine will send a $\Refute$ tag.
Here, we assume $\Learn_m$ has some mechanism of estimating improvement in
    predictive performance. From Table \ref{tab:tox1} we see that
    if this assessment is accurate. The $\PXPk$
    tag-sequence for the MLP will therefore be $\langle {\Init}_h,{\Revise}_m,{\Term}_{h/m} \rangle$ in
    71 out of 73 datasets, and $\langle {\Init}_h,{\Refute}_m,{\Term}_{h/m} \rangle$ in the
    2 of 73 cases. One-Way Intelligibility for the machine only follows in the former cases
    and not the latter. The corresponding numbers for the GNN are 
    63 and 10 respectively. These results highlight the
    following:
    (1) The authors in \citep{dash:botgnn,dash:drm}
    only point out unintelligibility of the machine-model for the human.
    However, with $\PXPk$, we can meaningfully
    talk about intelligibility of the human domain-knowledge for the machine; and
    (2) Intelligibility of human-knowledge for the machine depends
        on the machine (here used to include the representation language as well):
        this is apparent from the difference in numbers tabulated
        for the MLP.
        
The literature also contains studies
in which the machine refutes human explanations in logical forms other than
most-specific clauses.
For example, in the first report of the 
of a Robot Scientist in \citep{ross:robot}, the robot identifies incompleteness
in an existing human explanation of a metabolic network in yeast that is in
the form of a graph. Conversely,
it is also possible for a human to revise his or her
predictions on being provided with a machine-constructed logical
explanation. For example, in \citep{mugg:expl}, the models constructed by
an ML system are shown to improve human predictive performance under
some conditions. We conjecture that these cases can similarly
shown to be instances of One-Way Intelligibility
using $\PXPk$ (with tag-sequences
$\langle {\Init}_h , {\Refute}_m , {\Term}_{h/m} \rangle$ in the former and
$\langle {\Init}_m,{\Revise}_H , {\Term}_{h/m} \rangle$ in the latter). 

\subsection{Examples of Two-Way Intelligibility}
\label{sec:twowayex}

The most interesting demonstrations of Two-Way Intelligibility
arise when interactions between human and machine are sustained
beyond a single exchange. We have selected two such instances from
the literature.

\vspace*{0.3cm}
\noindent
{\bf {Case Study 3.}}
In \citep{moyle}, the authors describe an interactive, human-ML system, called ACUITY, for
detecting cyber-attacks. ACUITY is
a tool that allows the identification of malware, using a combination
of human-expertise and machine learning 
techniques. ACUITY is part of a larger software
environment that allows cybersecurity experts to query, consult and modify
a large (cloud-based) database of security logs, and incorporate specialised
knowledge of security threats. ACUITY itself is an extension of the Inductive Logic Programming (ILP)
engine Aleph \citep{aleph}, running in ``incremental'' mode. In this mode,
interaction commences with the human providing a data instance; followed
by iterations of the ILP engine constructing a hypothesis and the human providing
feedback to correct the hypothesis until the human is satisfied with the result.
Two key assumptions in the work appear to be:
(a) that the human's prediction is always correct; and
(b) there is `revision' by the human of his or her model.
Figure \ref{fig:acuity} shows a simple text-based version of the options available
to the human when providing feedback to ACUITY.
\begin{figure}[t]
{\footnotesize{
Top-level menu: 
   \begin{verbatim}
   Options:
        -> "ok." to accept default example
        -> Enter an example
        -> ctrl-D or "none." to end
    Response [default:none] ?
  \end{verbatim}
Second-level menu: (based on providing an example at the top-level menu and reply from ACUITY)
    \begin{verbatim}
   Options:
        -> "accept." to accept clause
        -> "prune." to prune clause and its refinements from the search
        -> "extent." to show the extent of the proposed hypothesis
        -> "db_extent." to show the extent of the proposed hypothesis
                        with respect to an external database
        -> "rebut." add rebuttal to the proposed hypothesis
        -> "constrain." add constaints from the proposed hypothesis
        -> "pick." add constaints from the most specific hypothesis
        -> "overgeneral." to add clause as a constraint
        -> "overgeneral because not(E)." to add E as a negative example
        -> "overspecific." to add clause as a positive example
        -> "overspecific because E." to add E as a positive example
        -> any Aleph command
        -> ctrl-D or "none." to end
    Response?
    \end{verbatim}
}}
    \caption{Text-based choices available
        to a human interacting with ACUITY. 
    }
    \label{fig:acuity}
\end{figure}

The reader will recognise that a number of options available to the human
have a close connection to the message-tags in $\PXPk$
For the purpose of this exercise, we will assume the following:
(a) If the human provides an example in the top-level menu then
    this maps the human sending a ${\Init}_h$ tag;
    Here the most-specific explanation is as in Case Study 2.
(b) If the human chooses to end the session then this maps
    to the human sending either  $\Ratify_h,\Term_h$ or $\Refute_h,\Term_h$ tags;
(c) If the human selects {\tt{accept}}, then this maps to
    human sending ${\Ratify}_h$;
(d) If the human selects any one of {\tt{prune}}, {\tt{rebut}}, {\tt{pick}},
    {\tt{constrain}}, or the {\tt{overgeneral}}/{\tt{overspecific}} options
    then this maps to the human sending ${\Refute}_H$; and
(e) Selections {\tt{extent}} and {\tt{db\_extent}} are not part of
    the $\PXPk$ protocol.
    
Each of the choices in (d) results in the ML engine automatically updating the background $B$
({\tt {rebut}} actually adds counter-examples, but this can be seen as adding
a constraint to $B$).
An inconsistency in the implementation choices made by the
ILP engine within ACUITY results in a difficulty
in exactly modelling the interaction.
The issue is the following: when
an example is entered (or re-entered) at the top-level menu, the
ILP engine automatically constructs (or re-constructs) a hypothesis
that agrees with the human's prediction. However the ILP engine
does not automatically construct (or re-construct) a hypothesis
when the human provides information at the second-level menu. 
Instead, the hypothesis is only constructed `on-demand'. The
reason for this lazy approach is unclear, but it is thought to be
one of efficiency. For our purposes, it is easier to consider
a slight alteration that does not affect correctness,
rectifies this, namely: the ILP engine automatically updates
its hypothesis whenever any information is provided by the human.
\begin{table}
\begin{center}
    \begin{tabular}{|c|c|c|} \hline
        Session & Expert Action & $\PXPk$\\
                &  (ACUITY)     &  Message Tag(s)\\ \hline
               &  Enter example &  ${\Init}_h$, ${\Revise}_m$ \\
                1& {\tt{rebut}}   & ${\Refute}_h$, ${\Revise}_m$ \\
                & {\tt{none}}    & ${\Ratify_h ,\Term}_h$ or $\Refute_h,\Term_h$ \\ \hline
              & Re-enter example & ${\Init}_h$, ${\Revise}_m$ \\
                2& {\tt{extent}}  &   -- \\
                & {\tt{pick}}    &  ${\Refute}_h$, ${\Revise}_m$ \\
                & {\tt{none}}    & ${\Ratify_h,\Term}_h$ or $\Refute_h,\Term_h$ \\ \hline
    \end{tabular}
\end{center}
    \caption{Two sessions with ACUITY. The cybersecurity expert provides an
        incident in the logs to be explained. The ML system (the ILP engine Aleph,
        with the modification described earlier)
        constructs a hypothesis. In this run, the expert finds the
        explanation incorrect. This results
        in new constraints being added (here shown through the use of {\tt{rebut}},
        and {\tt{pick}}). The corresponding
        ${\PXPk}$ messages are shown in the rightmost column.}
    \label{tab:acuitylog}
\end{table}

With this change, we are able to identify more
clearly the $\PXPk$
message sequences resulting from the human's interaction with ACUITY. 
Analogous to Case Study 2, the human providing an example $x$ along with
background knowledge $B$ can be seen as sending an ${Init}_h$ containing the
triple $(x,y_h,e_{y_h|x,B})$.\footnote{A correct analogy to Case Study 2 would
require representing the entries as $1$-dimensional vectors.}
The assumption that the human's prediction is always correct is
reflected in requiring that the prediction $y_m$ made by the machine for
$x$ matches the prediction $y$ by the human for $x$. Along
with the constraint of compatible agents, this means that 
it will always be the case that $\Match_m(y_m,y_h) = true$ 
and $\Match_h(y_h,y_m) = true$.
According to the definition of transitions in $\PXPk$
it is not difficult to see that these
constraints ensure: (i) that neither human nor machine ever sends a
$\Reject$ tag; (ii) the human can only send $\Init$, $\Ratify$, 
$\Refute$ or $\Term$ tags; and (iii) the machine can only send
$\Ratify$, $\Revise$ or $\Refute$ tags. It is not difficult
to see that there exist interactions that exhibit Two-Way Intelligibility.
Example of message-tag sequences with Two-Way Intelligibility are:
$\langle \Init_h, \Revise_m, \Ratify_h, \Term_h \rangle$; and
$\langle \Init_h, \Revise_m, \Refute_h, \Revise_m,$ $\Ratify_h, \Term_h \rangle$.

We examined logs of sessions with ACUITY.\footnote{
Data and annotated ACUITY logs kindly provided
by Steve Moyle, Amplify Intelligence UK and Cyber Security Centre,
University of Oxford.}
The sequence of selections made by
a cyber-security expert for two such sessions is shown in Table \ref{tab:acuitylog},
along with the $\PXPk$ message-tag sequence. Both sessions are classified
as being Two Way Intelligible by Def.~\ref{def:2wayintell} when the human ends
the session with $\Ratify_h,\Term_h$.

\noindent
In each of the sessions the machine
constructs a hypothesis, the human examines the prediction
and refutes its explanation. The 
machine revises in response. The interaction is very similar in spirit 
to a much older logic-based human-machine system which we consider next.

\vspace*{0.3cm}
\noindent
{\bf {Case Study 4.}}
To the best of our knowledge,  the most direct, long-running real-world example
of human-refutation of the logical explanation constructed by
a machine, which then results in revision of its hypothesis by the machine 
is from an early decision-support tool in chemical pathology.
PEIRS \citep{PEIRS:j:1993} was an extremely successful
decision-support tool for a pathologist.
The machine's hypothesis at any point in time was an ordered set of rules. PEIRS relied entirely on the ability of the pathologist
to read, understand and refute the model's prediction and explanation. The
explanation consisted of the sequence of rules used to arrive at the conclusion. The
refutation in turn triggers a revision to the machine's hypothesis. PEIRS 
was the first in a family of tools developed under the umbrella of 
``ripple-down rules'', or RDRs, which have continued to be deployed and used with great
commercial success \citep{Comp:Kang:b:2021}. Interactions for all such
systems consist of iterations of one or the other of the following:

\begin{enumerate}
    \item[A.] The machine proposes a prediction and an explanation for a data instance,
        and the human accepts the prediction and explanation; or
    \item[B.] The machine proposes a prediction and an explanation for a data instance,
        the human refutes the prediction or explanation or both, and provides
        corrective feedback. The machine (necessarily) revises its model
        to be consistent with
        the correction (the modification may add a new rule or change an existing
        one).
\end{enumerate}

It is not difficult to see that the interaction (A) will map to
the sequence $\langle {\Init}_m,{\Ratify}_h,$ ${\Term}_{h/m} \rangle$; and (B) will map
to $\langle {\Init}_m,{\Refute}_h,{\Revise}_m,$ ${\Ratify}_h,{\Term}_{h/m} \rangle$. More sophisticated implementations can also result in (B) containing multiple occurrences
of iterations of the $\Refute_h,\Revise_m$ pair, before ending in
$\Ratify_h,\Term{h/m}$. 
Both variants are within the scope of $\PXPk$ with a large enough value of $k$.
The reader will recognise that type (A) interactions constitute One-Way Intelligibility
for the human; and type (B) interactions constitute Two-Way Intelligibility.

Quantifiable evidence in PEIRS of type (B) interactions can
be obtained from the results reported in \citep{PEIRS:j:1993}, summarised here:

    \begin{itemize}
        \item The machine commenced operation with an initial model.
            containing approximately 200 rules
        \item Over a period of about 120 working days, the machine's entire
            model consisting of 950 rules was obtained by interaction with a pathologist
            with no prior programming skills, purely by
            performing type (B) interactions with the machine.
    
        \item The machine-based model's joint accuracy of prediction and
            explanation (``right for the right reasons'') was about
            92\%. The model was used routinely, interpreting around
            500 reports a day.
    \end{itemize}

\noindent
It is evident that the principal interaction mechanism for updating the PEIRS
model can be characterised as demonstrating Two-Way Intelligibility.
The techniques developed in PEIRS were commercialised by Pacific
Knowledge Systems and it was later acquired by Beamtree Holdings
Limited. In \citep{compton2013}, some details
are provided of Beamtree's RDR system which had 
acquired about 3000 rules and had interpreted biochemistry reports of about 7 million
patients over a period of about 9 years (by 2013)
. The machine-model continued
in use until October 2022, interpreting about 8000
reports a day.  The principal mechanism for rule acquisition in
this system remains type (B) interaction with
a human expert.\footnote{Based on data and logs kindly provided by
by Lindsay Peters,
Beamtree Holdings Ltd; and Paul Compton,
University of New South Wales.} 

\section{Related Work}
\label{sec:relwork}

In Section \ref{sec:reappraise} we referred to a number of sources that
are relevant to one or the other of the Intelligibility Axioms in
Section \ref{sec:axioms}. In this section we turn to some
other relevant work.

From the framework developed in
this paper intelligibility can be seen as a ternary
relation involving: the information-provider, the
information provided, and the information-recipient.
In the literature on Explainable ML,
this has re-emerged as an important requirement for
the acceptability of ML (see \citep{Mill:j:2019} for a recent example citing earlier work~\citep{Hilt:j:1990},
and~\citep{dm:ml} for an early identification of this).
Furthermore, the explainer and the explainee can be, at different times, the same person or agent.

A large literature has built up over recent years addressing the problems of inscrutable ``black box'' models generated by some modern machine learning techniques which then require mechanisms for ``explainability''~\citep{Guid:etal:j:2019}. There are several
excellent reviews available on Explainable Machine Learning (XML), and
we refer to the reader to them. More broadly, the origins of XML 
can be found in a prominent DARPA project, launched in 2016 titled ``Explainable Artificial Intelligence (XAI)''~\citep{DARP:m:2016} which is
credited with initiating efforts in XML~\citep{Adam:x:2022}. In fact,
XAI itself can be viewed as a continuation of pre-existing
research trends: earlier approaches in machine learning largely used models based on knowledge representations developed in artificial intelligence that were designed to be interpretable (such as rules, decision trees or logical theories), whereas only recently the drive for accuracy on ever-larger datasets has led to models that require explanation~\citep{Rudi:etal:j:2022}.

The DARPA project used the term ``explainability'' to denote the use of techniques to 
generate an explanation for a model~\citep{Gunn:Aha:j:2019}.
A distinction can therefore be made between \emph{interpretable} models, which are constrained to work in a way that is (at least, in principle) understandable to humans, and \emph{explainable} methods, which are applied to the results obtained from black box models~\citep{Rudi:etal:j:2022}.
This difference is similar to the distinction made between model transparency
and \textit{post hoc} explainability for black box models~\citep{Lipt:j:2018}. 
However, interpretability is not simply a property of a class of models, but also of the data, feature engineering, and other factors, and, as such, it may be difficult to achieve in applications~\citep{Lipt:j:2018,Rudi:etal:j:2022}.
Explainability by \textit{post hoc} methods suffers from problems of approximation, since the explainable model is not usually the model responsible for making the predictions.
Recent criticism of \textit{post hoc} explainability suggest such methods should not be adopted for certain medical~\citep{Babi:etal:j:2021} and legal~\citep{ValE:etal:j:2022} applications.
In particular, it appears that such explanations may not be able to satisfy the legislative requirements for avoidance of discrimination, for example~\citep{ValE:etal:j:2022}.

In presentations of the problem of explainability it is usually assumed
there is a (prototypical) human to which a machine is providing the explanation, in the terminology we have used in this paper, this means One-Way Intelligibility for the human.
Even in this limited setting, considering only \emph{techniques} for explanation
ignores other issues, like the role of the explainee,
for whom the explanation is intended, which is contrary to evidence from social science~\citep{Mill:j:2019}.
For instance, the diversity of explainees suggests thinking of explanation
in terms of the understanding of the explainee~\citep{Soko:Flac:x:2021}.
From this standpoint~\citep{Soko:Flac:p:2018,Soko:Flac:x:2021} propose that
explanation should be a
bi-directional process rather than a one-off, one-way delivery of
an explanation to a recipient. Therefore explainability is a process involving reasoning over interpretable insights that are
aligned with  the explainee's background knowledge.
The proposal is that an explainer should allow explainees to interact and rebut explanations; in the proposed framework the entire explanatory process
is based on a reasoning system enabling communication between agents. Such
a framework can be seen as fulfilling many of the requirements in \citep{jmv:human_ml}
for human-centred machine-learning that is intelligible to the (specific) human.

The natural way to view communication between human agents is 
as a dialogue. We have examined the relation of one
such model proposed in \citep{madumal} in Sec. \ref{sec:correct_dial}.
Full dialogue models between agents attempt to characterise multiple kinds of interaction
like questions, explanations, refutations, clarification, argumentation
and the like \citep{agentgames}. 
We are aware of at least two threads of work in the ML literature that
recognises some of the aspects just described, wthout necessarily formulating
it as a dialogue model. In Argument-Based ML, or ABML~\citep{Zabk:etal:c:2006,Mozi:etal:j:2007}, data provided by a human
to an ML engine includes explanations formulated in terms of background knowledge shared between human and machine. These explanations constitute `arguments'
for the label for data instances, and the ML system attempts to construct
hypotheses that are consistent with the explanations. Learning from
explanations was employed by early ML systems like MARVIN \citep{sammut:marvin}
that performed supervised-learning in the original sense of a human guiding
the construction of hypotheses by a machine. The learning protocol employed
in such systems is naturally interactive, involving human-presentation of
data along with explanations (if any), and revision of hypotheses by machine.
The interaction can result in Two-Way Intelligibility if
interactions consist of human refutation of a machine's explanations; and the
machine revises its hypothesis in response. However, no
refutation  of the human's explanation is envisaged within ABML
nor any revision of his or her hypothesis. 

In work on XIL applications to computer vision, where the learner is implemented by a deep network, the key intermediate step of defining ``concepts'' was implemented to better communicate with the human to enable any required revisions~\citep{Stam:etal:p:2021}.\footnote{Such \emph{concept-based} explainability methods have also been applied more widely for deep learning~\citep{Yeh:etal:c:2022}.
For example, concept-based explanations have recently been applied
in an attempt to comprehend the chess knowledge
obtained by AlphaZero, and concluded that this approach did reveal a number of relationships between
learned concepts and historical game play, from standard opening moves to tactical skills, as assessed by
a former world chess champion~\citep{McGr:etal:x:2022}.}

The characterisation  predictions and explanations in XIL is similar
in spirit to the categories defined by the guard functions in this paper.
As with ABML, an XIL system demonstrates Two-Way Intelligibility
when the human provides refutations (in XIL, this is done by
augmenting the data) and the machine revises its hypothesis. Also
in common with ABML, the machine does not
provide refutations for the human's prediction and/or explanation,
and revision of the human's hypothesis is not envisaged. Finally,
although not cast either as ABML or XIL, but nevertheless related to
aspects of both are implementations of incremental
learning designed human-machine collaboration.
An example of using a combination of neural and symbolic learning for collaborative
decision-making systems for medical images \citep{Schm:Finz:p:2020}). We conjecture
that techniques such as this exhibit at least One-Way Intelligibility for
the machine, since it relies on the human providing refutations
and the machine improving its performance as a result.

\section{Concluding Remarks}
\label{sec:concl}
In this paper we have sought to understand intelligibility requirements on interaction
between human and machine-learning systems by abstracting to a broader computational
setting of communication between agents that make predictions and provide explanations.
The paper is thus not about developing particular techniques for prediction or explanation,
but on the identification of some general principles for detecting intelligible
interaction between agents. This abstraction necessarily entails some
model for interaction. Here, we model agents as automata with some special
characteristics, and their interaction follows a communication protocol. Intelligibility is then
defined as a property defined on the result of executing communication the protocol. 
There are some advantages to adopting this approach. First, as with any
mathematical formalisation, we are able to focus on a setting where
aspects of the bigger questions of what is and is not intelligible can be answered unambiguously. Secondly, we are able to use the abstraction to clarify the working of existing implementations.
As an examples of the former, we are able to define concepts of One-Way and Two-Way
Intelligibility between agents purely syntactically from the messages sequences sent
by the corresponding automata, and identify conditions (`Actuation Constraints') under
which this syntactic property would correctly capture some fairly natural
semantic notions of intelligibile interaction between a human and
machine-learning system. As examples of the latter, we show that we are able
to prove that a complex dialogue model for question-answering using natural language can
be mapped to message-sequences in our model; and separately, provide case-studies of
One- and Two-Way Intelligibility in human-in-the-loop implementations that have used
formal logic-based representations for explanations. 

There is broad consensus that
intelligible communication between a machine-learning system and a human
depends on the ML engine, the human and information sent from one to the other; and
a recognition that for complex problems it is important to consider intelligibility to
the machine of information provided by a human.
There is also a long history of formal models for `legal' communication as a
(mathematical) relation between the sender, receiver and the messages exchanged. Surprisingly,
little attention has been paid in bringing
these two aspects together. 
There are of course limits to which an approach purely based on the syntax of
messages--such as the one here--can be used to identify a complex concept like intelligibility.
Some of these can be addressed partially by making the guarded transitions more elaborate,
requiring the $\PEX$-functions to capture the semantics, and extending to a multi-valued logic.
But it is not coincidental that the R's we have identified for detecting
intelligibility --- $\Refute$, $\Revise$, $\Ratify$, and $\Reject$ --- are at
the heart of advancing understandability in Science. We suggest
they may play a similar role in evolving a shared understanding of
data by human-machine systems of the kind proposed in \cite{krenn:nature2022}.
More generally, we envisage  `intelligibility protocols' like
$\PXP$ as being an integral part of 
the design and analysis of Explainable AI (XAI) systems from the ground-up.

\vspace*{0.5cm}
{\scriptsize{
\noindent
{\bf{Acknowledgements:}} AS is a Visiting Professor at Macquarie University, Sydney
and a Visiting Professorial Fellow at UNSW, Sydney. He is also the Class of 1981
Chair Professor at BITS Pilani, Goa,
the Head of the Anuradha and Prashant Palakurthi Centre for AI Research (APPCAIR) at BITS Pilani, and 
TCS Affiliate Professor. 
MB acknowledges support in part by Rich Data Co. Pty and the Australian Government's
Innovations Connections scheme (awards ICG001855 and ICG001858).
 EC is supported by an NHMRC investigator grant.}
 The authors would like to thank  Tirtharaj Dash, Rishabh Khincha,
Soundarya Krishnan, Steve Moyle, Lindsay Peters, and Paul Compton for kindly providing
transcripts of experiments that were used for case studies reported in this paper. AS and MB owe a debt of gratitude to Donald Michie, who shaped much of their
thinking on the views expressed in this paper.
}

\appendix
\normalsize
\section{Formal Model for $\PXP$}
\label{app:lxp}

We formalise $\PXP$ as a protocol for communicating between
$\PEX$ agents modelled as automata. 
Let $S_{mn}(x)$ denote a session between automata $a_m$ and $a_n$
in about a data-instance.\footnote{
There may be multiple sessions between $a_{m}$ and $a_n$ involving
the same data-instance, and we would need an additional index to
capture this. We ignore this here, and also omit $x$ when the
context is obvious.} We adopt the convention that the session
$S_{mn}$ is initiated by $a_m$.

$S_{mn}$ can be represented by the execution of the protocol
which results in a sequence of configurations.
$\langle \gamma_{mn,1},\gamma_{mn,2},\ldots, \gamma_{mn,k} \rangle$.
It is helpful to think of any configuration $\gamma_{mn,i}$ as being
composed of the pair $(\gamma_{m,i},\gamma_{n,i})$.
Here,  $\gamma_{m,i}$ is the (local) configuration of automaton $a_m$
and $\gamma_{n,i}$ is the configuration of $a_n$.

\begin{mydefinition}[Local Configuration]
We define $\gamma_{n,i}$ =  $(s_{n,i},(H_{n,i},D_{n,i}),\mu_{n,i})$. Here the
$s_{\cdot,i}$ is a state;
$H_{\cdot,i}$ is a hypothesis;
$D_{\cdot,i}$ is a set of 4-tuples; 
$\mu_{\cdot,i}$ is the message sent or received.
From the grammar rules,
messages are of the form $+(A,(t,(x,y,e)))$
or $-(A,(t,(x,y,e))$, where $A$ is either $m$ or $n$ (denoting
$a_m$ or $a_n$ for short);
$t$ is a message-tag,
$x$ is a data-instance, $y$ is a label, and $e$ is an explanation.
The corresponding local configuration $\gamma_{n,i}$ is similar.
\end{mydefinition}

\subsection{Guards and Guarded Transitions}

Before defining transitions between configurations, 
we introduce the guards here. The guard $\top$ is 
trivially true in all configurations.
The definitions of non-trivial guards are the
same for all ${\PEX}$ agents, and
we define them here for the receiving automaton ($a_n$).

\begin{mydefinition}[Guards]
 \label{def:guards}
 Let $a_n$ be a ${\PEX}$ agent.
Let $\gamma_{mn,i} = (\gamma_{m,i},\gamma_{n,i})$ be a configuration in
a session $S_{mn}$,
where $\gamma_{m,i} = (s_{m,i},(H_{m,i},D_{m,i}),\mu_{m,i})$ and
$\gamma_{n,i} = (s_{n,i},(H_{n,i},D_{n,i}),\mu_{n,i})$.
Let $\mu_{m,i}$ = $+(n,(t_m,(x,y_m,e_m)))$,
$\mu_{n,i}$ = $-(m,(t_m,(x,y_m,e_m)))$,
$y_n = {\mathtt{PREDICT}}_n(x,H_{n,i})$, and
$e_n = {\mathtt{EXPLAIN}}_n((x,y_n),H_{n,i})$.
Then we define the guards:

\begin{itemize}
    \item[$g_1$:] ${\mathtt{MATCH}}_n(y_n,y_m)$ $\wedge$ ${\mathtt{AGREE}}_n(e_n,e_m)$
    \item[$g_2$:] ${\mathtt{MATCH}}_n(y_n,y_m)$ $\wedge$ $\neg{\mathtt{AGREE}}_n(e_n,e_m)$
    \item[$g_3$:] $\neg{\mathtt{MATCH}}_n(y_n,y_m)$ $\wedge$ ${\mathtt{AGREE}}_n(e_n,e_m)$
    \item[$g_4$:] $\neg {\mathtt{MATCH}}_n(y_n,y_m)$ $\wedge$ $\neg {\mathtt{AGREE}}_n(e_n,e_m)$
\end{itemize}
\end{mydefinition}

\begin{myremark}
\label{rem:guards}
We note the following: 
\begin{itemize}
    \item[(a)] At most one of the four guards can be
            true in a configuration $\gamma_{mn,i}$.
    \item[(b)] The automaton in the ${\mathtt{CAN\_SEND}}$
        state in $\gamma_{mn,i}$ (here $a_n$) now checks the guards to decide on the message-tag;
    \item[(c)]  The guards for all automata use only the {$\PEX$} functions ${\mathtt{MATCH}}$
            and ${\mathtt{AGREE}}$. However, since these functions
            are automaton-specific, the value of the guard
            function in one automaton may or may not agree with the corresponding value in
            a different automaton. This can result in complex interaction patterns, some
            of which can be counter-intuitive. In this paper, we will mainly
            restrict ourselves to {\em compatible\/} automata. As will
            be seen below, this results in a substantial simplification in the set
            of messages possible. 
        
\end{itemize}
\end{myremark}

To address (c) we focus on the special case where a pair
of $\PEX$ agents that agree with each other on their ${\mathtt{MATCH}}$ and
${\mathtt{AGREE}}$ functions within a session.

\begin{mydefinition}[Compatible Automata]
Let $S_{mn}$ be a session between ${\PEX}$ agents $a_m$ and $a_n$.
Let $Y_m = \{y_m: +(n,(\cdot,(x,y_m,\cdot))\}$ be the set of predictions in
messages sent by $a_m$ to $a_n$ and
$Y_n = \{y_n: +(m,(\cdot,(x,y_n,\cdot))\}$ be the set of
messages sent by $a_n$ to $a_m$.
Let $E_m = \{e_m: +(n,(\cdot,(x,\cdot,e_m))\}$ be the set of explanations in
messages sent by $a_m$ to $a_n$ and
$E_n = \{e_n: +(m,(\cdot,(x,\cdot,e_n)))\}$ be the set of explanations in 
messages sent by $a_n$ to $a_m$.
We will say there is a functional agreement on predictions between
$a_m$ and $a_n$ in $S_{mn}$, or
$a_m~\simeq_y a_n$ in $S_{mn}$, 
if for all $y_m \in Y_m$ and $y_n \in Y_n$,
${\mathtt{MATCH}}_m(y_m,y_n)$ $=$ ${\mathtt{MATCH}}_n(y_n,y_m)$.
Similarly we will say there is a functional agreement on explanations between
$a_m$ and $a_n$ in $S_{mn}$, or $a_m \simeq_e a_n$ in $S_{mn}$, if for all $e_m \in E_m$
and $e_n \in E_n$ 
${\mathtt{AGREE}}_m(e_m,e_n)$ $=$ ${\mathtt{AGREE}}_n(e_n,e_m)$.
We will say automata $a_m$ and $a_n$ are compatible in session $S_{mn}$
iff $a_m \simeq_y a_n$ and $a_m \simeq_e a_n$ in $S_{mn}$.
\end{mydefinition}
\noindent
We assume that the oracle $\Delta$ is compatible with any $\PEX$ agent. 

\subsection{Interaction Between Automata}

Let us assume $a_m$ sends a message $\mu$ to $a_n$.
 The response from $a_n$ is determined by the definition of
 a guarded transition relation.
Each element of this relation contains a guard $g$.
Intuitively, $a_n$ performs a computation $\Pi$ on its current configuration
$\gamma_n$; evaluates $g$; and then sends a response $\mu^\prime$ to $a_m$.
That is, the transition relation can be specified as a set
of 4-tuples $(\gamma,\Pi,g,\gamma^\prime)$, where $\gamma,\gamma^\prime$ are
global configurations, consisting of local configurations for $a_m$ and $a_n$.
($\gamma$ = $(\gamma_m,\gamma_n)$ and $\gamma^\prime$ =
$(\gamma^\prime_m,\gamma^\prime_n)$).

The message-tag in $\mu^\prime$ associated with an element in the transition
relation are obtained from the following categorisation:

{\small{
\begin{center}
\begin{tabular}{cc|c|c|c}
&\multicolumn{1}{c}{\mbox{}}&\multicolumn{2}{c}{Explanation} & \\
&\multicolumn{1}{c}{\mbox{}}&\multicolumn{1}{c}{ ${\mathtt{AGREE}}(e_n,e_m)$}&\multicolumn{1}{c}{$\neg {\mathtt{AGREE}}(e_n,e_m)$} & \\ \cline{3-4}
& ${\mathtt{MATCH}}(y_n,y_m)$ & $\Ratify$ & $\Refute$ or & \\
Prediction & & (A) & $\Revise$ (B) & \\ \cline{3-4}
&  $\neg {\mathtt{MATCH}}(y_n,y_m)$ & $\Refute$ or & $\Reject$ &  \\
& & $\Revise$ (C) & (D) & \\ \cline{3-4}
\end{tabular}
\end{center}
}}

Note that in category (A), guard $g_1$ (see Def. ~\ref{def:guards}) will be true; 
in category (B), guard $g_2$ will be true;
in category (C), guard $g_3$ will be true; and
in category (D), guard $g_4$ will be true. In
addition, a further test ($g^\prime$ in Def. ~\ref{def:trans}) below will be used in (B) and (C) 
to decide on whether a $Refute$ or $Revise$ message-tag is sent.

\begin{mydefinition}[Guarded Transition Relation]
\label{def:trans}
Let $S_{mn}$ be a session between 
automata $a_m$ and $a_n$, where $a_m$ sends a message
to $a_n$.
The transition relation for $S_{mn}$ is the
 set of 4-tuples $(\gamma,\Pi,g,\gamma^\prime)$.
 Let $\gamma$ = $((s_m,(H_m,D_m),+(n,\mu)),(s_n,(H_n,D_n),-(m,\mu)))$
and $\gamma^\prime$ = $(s^\prime_m,((H_m,D_m),-(n,\mu^\prime)),\allowbreak (s^\prime_n,(H^\prime_n,D^\prime_n),+(m,\mu^\prime)))$, where
$s_m = {\mathtt{CAN\_RECEIVE}}$,
$s_n = {\mathtt{CAN\_SEND}}$;
$s^\prime_m = {\mathtt{CAN\_SEND}}$,
$s^\prime_n = {\mathtt{CAN\_RECEIVE}}$.
Let $\Pi$ =
($D^\prime_n:=D_n \cup \{(x,y_m,e_m,m)\})~;~P~;$ $~y_n:= {\mathtt{PREDICT}}(x,H_{n})~;~$
$e_n:= {\mathtt{EXPLAIN}}((x,y_n),H_{n})~;~
y^\prime_n:={\mathtt{PREDICT}}(x,H^\prime_n)~;~e^\prime_n:=
{\mathtt{EXPLAIN}}((x,y^\prime_n),H^\prime_n)$).
Let $g^\prime$ =
${\mathtt{MATCH}}(y^\prime_n,y_m)~\allowbreak \wedge~{\mathtt{AGREE}}(e^\prime_n,e_m)$.

For compactness, we only tabulate $\mu$, $P$, $g$, and $\mu^\prime$.
This is shown in Table \ref{tab:allgt}.
\end{mydefinition}

{\small{
\begin{table}
\begin{center}
  \begin{tabular}{|l|l|l|c|l|} \hline
Trans& \multicolumn{1}{|c|}{$\mu$ (received by $a_n$)} & \multicolumn{1}{|c|}{$P$} & {$g$} & \multicolumn{1}{|c|}{$\mu^\prime$ (sent by $a_n$)}\\ \hline
0. & No message & $H^\prime_n := H_n$& $\top$ & $(Init,(x,y^\prime_n,e^\prime_n))$\\
\hline
1. & $(Init, (x,y_m,e_m))$  & $H^\prime_n := H_n$  & $ g_1 $ & $(Ratify, ((x,y^\prime_n,e^\prime_n))$  \\
\hline
2. & $(Init, (x,y_m,e_m))$  & $H^\prime_n := {\mathtt{LEARN}}(H_n,D^\prime_n)$  & $ g_2~\wedge~\neg g^\prime $ & $(Refute, ((x,y^\prime_n,e^\prime_n))$  \\
\hline
3. & $(Init, (x,y_m,e_m))$  & $H^\prime_n := {\mathtt{LEARN}}(H_n,D^\prime_n)$  & $ g_2~\wedge~g^\prime $ & $(Revise, ((x,y^\prime_n,e^\prime_n))$  \\
\hline
4. & $(Init, (x,y_m,e_m))$  & $H^\prime_n := {\mathtt{LEARN}}(H_n,D^\prime_n)$  & $ g_3~\wedge~\neg g^\prime $ & $(Refute, ((x,y^\prime_n,e^\prime_n))$  \\
\hline
5. & $(Init, (x,y_m,e_m))$  & $H^\prime_n := {\mathtt{LEARN}}(H_n,D^\prime_n)$  & $ g_3~\wedge~g^\prime $ & $(Revise, ((x,y^\prime_n,e^\prime_n))$  \\
\hline
6. & $(Init, (x,y_m,e_m))$  & $H^\prime_n := H_n$  & $ g_4 $ & $(Reject, ((x,y^\prime_n,e^\prime_n))$  \\
\hline
7. & $(Ratify, (x,y_m,e_m))$  & $H^\prime_n := H_n$  & $ g_1 $ & $(Ratify, ((x,y^\prime_n,e^\prime_n))$  \\
\hline
8. & $(Ratify, (x,y_m,e_m))$  & $H^\prime_n := {\mathtt{LEARN}}(H_n,D^\prime_n)$  & $ g_2~\wedge~\neg g^\prime $ & $(Refute, ((x,y^\prime_n,e^\prime_n))$  \\
\hline
9. & $(Ratify, (x,y_m,e_m))$  & $H^\prime_n := {\mathtt{LEARN}}(H_n,D^\prime_n)$  & $ g_2~\wedge~g^\prime $ & $(Revise, ((x,y^\prime_n,e^\prime_n))$  \\
\hline
10. & $(Ratify, (x,y_m,e_m))$  & $H^\prime_n := {\mathtt{LEARN}}(H_n,D^\prime_n)$  & $ g_3~\wedge~\neg g^\prime $ & $(Refute, ((x,y^\prime_n,e^\prime_n))$  \\
\hline
11. & $(Ratify, (x,y_m,e_m))$  & $H^\prime_n := {\mathtt{LEARN}}(H_n,D^\prime_n)$  & $ g_3~\wedge~g^\prime $ & $(Revise, ((x,y^\prime_n,e^\prime_n))$  \\
\hline
12. & $(Ratify, (x,y_m,e_m))$  & $H^\prime_n := H_n$  & $ g_4 $ & $(Reject, ((x,y^\prime_n,e^\prime_n))$  \\
\hline
13. & $(Refute, (x,y_m,e_m))$  & $H^\prime_n := H_n$  & $ g_1 $ & $(Ratify, ((x,y^\prime_n,e^\prime_n))$  \\
\hline
14. & $(Refute, (x,y_m,e_m))$  & $H^\prime_n := {\mathtt{LEARN}}(H_n,D^\prime_n)$  & $ g_2~\wedge~\neg g^\prime $ & $(Refute, ((x,y^\prime_n,e^\prime_n))$  \\
\hline
15. & $(Refute, (x,y_m,e_m))$  & $H^\prime_n := {\mathtt{LEARN}}(H_n,D^\prime_n)$  & $ g_2~\wedge~g^\prime $ & $(Revise, ((x,y^\prime_n,e^\prime_n))$  \\
\hline
16. & $(Refute, (x,y_m,e_m))$  & $H^\prime_n := {\mathtt{LEARN}}(H_n,D^\prime_n)$  & $ g_3~\wedge~\neg g^\prime $ & $(Refute, ((x,y^\prime_n,e^\prime_n))$  \\
\hline
17. & $(Refute, (x,y_m,e_m))$  & $H^\prime_n := {\mathtt{LEARN}}(H_n,D^\prime_n)$  & $ g_3~\wedge~g^\prime $ & $(Revise, ((x,y^\prime_n,e^\prime_n))$  \\
\hline
18. & $(Refute, (x,y_m,e_m))$  & $H^\prime_n := H_n$  & $ g_4 $ & $(Reject, ((x,y^\prime_n,e^\prime_n))$  \\
\hline
19. & $(Revise, (x,y_m,e_m))$  & $H^\prime_n := H_n$  & $ g_1 $ & $(Ratify, ((x,y^\prime_n,e^\prime_n))$  \\
\hline
20. & $(Revise, (x,y_m,e_m))$  & $H^\prime_n := {\mathtt{LEARN}}(H_n,D^\prime_n)$  & $ g_2~\wedge~\neg g^\prime $ & $(Refute, ((x,y^\prime_n,e^\prime_n))$  \\
\hline
21. & $(Revise, (x,y_m,e_m))$  & $H^\prime_n := {\mathtt{LEARN}}(H_n,D^\prime_n)$  & $ g_2~\wedge~g^\prime $ & $(Revise, ((x,y^\prime_n,e^\prime_n))$  \\
\hline
22. & $(Revise, (x,y_m,e_m))$  & $H^\prime_n := {\mathtt{LEARN}}(H_n,D^\prime_n)$  & $ g_3~\wedge~\neg g^\prime $ & $(Refute, ((x,y^\prime_n,e^\prime_n))$  \\
\hline
23. & $(Revise, (x,y_m,e_m))$  & $H^\prime_n := {\mathtt{LEARN}}(H_n,D^\prime_n)$  & $ g_3~\wedge~g^\prime $ & $(Revise, ((x,y^\prime_n,e^\prime_n))$  \\
\hline
24. & $(Revise, (x,y_m,e_m))$  & $H^\prime_n := H_n$  & $ g_4 $ & $(Reject, ((x,y^\prime_n,e^\prime_n))$  \\
\hline
25. & $(Reject, (x,y_m,e_m))$  & $H^\prime_n := H_n$  & $ g_1 $ & $(Ratify, ((x,y^\prime_n,e^\prime_n))$  \\
\hline
26. & $(Reject, (x,y_m,e_m))$  & $H^\prime_n := {\mathtt{LEARN}}(H_n,D^\prime_n)$  & $ g_2~\wedge~\neg g^\prime $ & $(Refute, ((x,y^\prime_n,e^\prime_n))$  \\
\hline
27. & $(Reject, (x,y_m,e_m))$  & $H^\prime_n := {\mathtt{LEARN}}(H_n,D^\prime_n)$  & $ g_2~\wedge~g^\prime $ & $(Revise, ((x,y^\prime_n,e^\prime_n))$  \\
\hline
28. & $(Reject, (x,y_m,e_m))$  & $H^\prime_n := {\mathtt{LEARN}}(H_n,D^\prime_n)$  & $ g_3~\wedge~\neg g^\prime $ & $(Refute, ((x,y^\prime_n,e^\prime_n))$  \\
\hline
29. & $(Reject, (x,y_m,e_m))$  & $H^\prime_n := {\mathtt{LEARN}}(H_n,D^\prime_n)$  & $ g_3~\wedge~g^\prime $ & $(Revise, ((x,y^\prime_n,e^\prime_n))$  \\
\hline
30. & $(Reject, (x,y_m,e_m))$  & $H^\prime_n := H_n$  & $ g_4 $ & $(Reject, ((x,y^\prime_n,e^\prime_n))$  \\
\hline
31.& $(Ratify,(x,y_m,e_m))$ &$H^\prime_n := H_n$ & $\top$ &$(Term,(x,y^\prime_n,e^\prime_n))$\\
\hline
 32.& $(Refute,(x,y_m,e_m))$ &$H^\prime_n := H_n$ & $\top$ &$(Term,(x,y^\prime_n,e^\prime_n))$\\
\hline
 33.& $(Revise,(x,y_m,e_m))$ &$H^\prime_n := H_n$ & $\top$ &$(Term,(x,y^\prime_n,e^\prime_n))$\\
\hline
 34.& $(Reject,(x,y_m,e_m))$ &$H^\prime_n := H_n$ & $\top$ &$(Term,(x,y^\prime_n,e^\prime_n))$\\
 \hline
  35.& $(Init,(x,y_m,e_m))$ &$H^\prime_n := H_n$ & $\top$ &$(Term,(x,y^\prime_n,e^\prime_n))$\\
 \hline
  36.& $(Init,(x,?,?))$ &$H^\prime_n := H_n$ & $\top$ &$(Refute,(x,y^\prime_n,e^\prime_n))$\\
 \hline
  37.& $(Init,(x,y,?))$ &$H^\prime_n := H_n$ & $\top$ &$(Refute,(x,y^\prime_n,e^\prime_n))$\\
 \hline
  38.& $(Init,(x,?,e_m))$ &$H^\prime_n := H_n$ & $\top$ &$(Refute,(x,y^\prime_n,e^\prime_n))$\\
 \hline
  39.& $(Term,(x,y_m,e_m))$ &$H^\prime_n := {\mathtt{LEARN}}(H_n,D^\prime_n)$ & $\top$ & No message \\ 
\hline
40. & No message & $H^\prime_n := H_n$& $\top$ & $(Term,(x,y^\prime_n,e^\prime_n))$\\
 \hline 
 \end{tabular}
\end{center}
\caption{Elements in the set comprising the guarded transition relation.}
\label{tab:allgt}
\end{table} 
}}
\subsubsection{The Special Case of Compatible Automata}
\begin{myproposition}
\label{prop:illegal}
Let $\psi$ be an execution of the $\PXP$ protocol in a
session between compatible agents $a_m$ and $a_n$. 
\end{myproposition}
\begin{itemize}
\item The transitions correspond to rows 8,9,10,11 and 12 will not occur in $\psi$.
\item The transitions correspond to rows 20,21,22,23 and 24 will not occur in $\psi$.
\item The transitions correspond to rows 25,26,27,28 and 29 will not occur in $\psi$. 
\end{itemize}
\begin{proof} 
We assume each of the above transitions represents a communication of the message 
$\mu^\prime=(t^\prime, (x,y^\prime_n,e^\prime_n))$. Without loss of generality, we 
assume that $n$ is the sender and $m$ is the receiver of this communication. 
We observe that the message tag in the above communication is different from {\tt{Init}}.
So it is immediately preceded by other transitions.
These preceding transitions would represent communications in which the agent 
$m$ sends the message $\mu=(t,(x,y_m,e_n))$ and the agent $n$ receives it. 
Let $H_m$ and $H_n$ be the hypotheses of  $m$ and $n$ before communicating the 
message $\mu$, $H^\prime_m$ and $H^\prime_n$ be the hypotheses of $m$ and $n$ 
after communicating $\mu^\prime$. $H^\prime_n$ be the updated hypothesis of $n$ 
after receiving the message $m$. Also observe that $y_m:= {\mathtt{PREDICT}}(x,H_{m})$,
$e_m:= {\mathtt{EXPLAIN}}((x,y_m),H_{m})$, $y^\prime_n:={\mathtt{PREDICT}}(x,H^\prime_n)$
and $~e^\prime_n:={\mathtt{EXPLAIN}}((x,y^\prime_n),H^\prime_n)$.

\begin{center}
\begin{tikzpicture}[->,>=stealth',shorten >=1pt,auto,node distance=0.8cm,semithick,inner sep=3pt]
\node[rectangle split, rectangle split parts=2] (A) {$H_m,y_m,e_m$ \nodepart{second} $H_n,y_n,e_n$};
\node (1) [right=of A]{};
\node (2) [right=of 1]{};
\node[rectangle split, rectangle split parts=2]  (B) [right=of 2] {$H_m,y_m,e_m$ \nodepart{second}$H^\prime_n,y^\prime_n,e^\prime_n$};
\node (3) [right=of B] {};
\node (4) [right=of 3] {};
\node[rectangle split, rectangle split parts=2]  (C) [right=of 4] {$H^\prime_m,y^\prime_m,e^\prime_m$ \nodepart{second}$H^\prime_n,y^\prime_n,e^\prime_n$};
\path (A) edge node[above]{$m$ sends $\mu$} node[below]{$n$ receives $\mu$} (B);
\path (B) edge node[above]{$n$ sends $\mu^\prime$} node[below]{$m$ receives $\mu^\prime$} (C);
\node[rectangle split, rectangle split parts=3] (D) [below=of 2] {$n$ checks the guards using
\nodepart{second}$\mathtt{MATCH}_n(y_n,y_m)$, ${\mathtt{AGREE}}_n(e_n,e_m)$
\nodepart{third} $\mathtt{MATCH}_n(y^\prime_n,y_m)$,  ${\mathtt{AGREE}}_n(e^\prime_n,e_m)$};
\node[rectangle split, rectangle split parts=3] (E) [below=of 4] {$m$ checks the guards using
\nodepart{second}$\mathtt{MATCH}_m(y_m,y^\prime_n)$, ${\mathtt{AGREE}}_m(e_m,e^\prime_n)$
\nodepart{third} $\mathtt{MATCH}_m(y^\prime_m,y^\prime_n)$,  ${\mathtt{AGREE}}_m(e^\prime_m,e^\prime_n)$};
\end{tikzpicture}
\end{center}

\begin{itemize}
\item Let us consider first the transitions 8,9,10,11 and 12 for communicating $\mu^\prime$. 
It is clear that the message tag received in each of these transitions  is {\tt{Ratify}}. 
It means the previous transition for communicating $\mu$ should be any one of 1,7,13,19 
and 25 and all of them have guard $g_1= \mathtt{MATCH}_n(y_n,y_m) \land {\mathtt{AGREE}}_n(e_n,e_m)$ 
is true. Since there is no change in the hypothesis of the agent $m$ and this is a 
collaborative session, the guard $\mathtt{MATCH}_m(y_m,y_n) 
\land {\mathtt{AGREE}}_m(e_m,e_n)$ is also true. Hence the guards for 8,9,10,11 
and 12 are all false and these transitions will never occur in any collaborative 
session using the {$\PXP$} protocol.

\item Now let us consider the transitions 20,21,22,23 and 24. It is clear that
the message tag $t$ in each of these transitions is {\tt{Revise}}. It means the 
previous transition should be any one of 3,5,9,11,15,17,21,23,27 and 29. In all theses
transitions, the guard $g'$ (which is $\mathtt{MATCH}_n(y^\prime_n,y_m) \land {\mathtt{AGREE}}_n(e^\prime_n,e_m))$)
is true. Since this is a collaborative session, the guard $\mathtt{MATCH}_m(y_m,y^\prime_n) 
\land {\mathtt{AGREE}}_m(e_m,e^\prime_n)$ is also true. Hence the guards for 20,21,22,23 
and 24 are all false and these transitions will never occur in any collaborative 
session using the {$\PXP$} protocol.

\item Finally, let us consider the transitions 25,26,27,28 and 29. It is clear that
the message tag $t$ in each of these transitions is {\tt{Reject}}. It means the 
previous transition should be any one of 6,12,18,24 and 30. Also there is no change in 
the hypothesis of the agent $m$ and the previous 
transition's guard $\lnot\mathtt{MATCH}_n(y_n,y_m) \land \lnot{\mathtt{AGREE}}_n(e_n,e_m)$ 
is true. Since this is a collaborative session, the guard $\lnot\mathtt{MATCH}_m(y_m,y_n) 
\land \lnot{\mathtt{AGREE}}_m(e_m,e_n)$ is also true. Hence the guards for 25,26,27,28 
and 29 are all false and these transitions will never occur in any collaborative 
session using the {$\PXP$} protocol.

\qed
\end{itemize}
\end{proof}

\begin{myremark}
\label{rem:termloop}
Let $S_{mn}$ be a session between compatible automata $a_m$ and $a_n$ consisting of a
sequence of configurations $\langle \gamma_1,\gamma_2,\ldots, \gamma_k \rangle$,
where $\gamma_i = (\gamma_{m,i},\gamma_{n,i})$. Let
$\gamma_{m,i} = (\cdot,\cdot,\mu_{m,i})$,
$\gamma_{n,i} = (\cdot,\cdot,\mu_{n,i})$ where
$\mu_{m,1}$ = $+(n,(Init,\cdot))$ and $\mu_{\cdot, k}$ = $+(\cdot,(Term,\cdot))$.
\begin{itemize}
\item If there is an $i$ such that the message tag in $\mu_{m,i}$ or $\mu_{n,i}$ is 
$\tt{Ratify}$, then for each $j$ such that $i<j<k$, $r_j=r_{j-1}$.  
\item If there is an $i$ such that the message tag in $\mu_{m,i}$ or $\mu_{n,i}$ is 
$\tt{Reject}$, then for each $j$ such that $i<j<k$, $r_j=r_{j-1}$.  
\end{itemize}
\end{myremark}

\begin{myremark}
There are only four transitions are enabled with `?' in the message: 
\begin{itemize}
\item The transition in which an agent initiate a session with `?'  as a prediction and/or an explanation.  For this, 
the transition in the 0th row (which is enabled as the guard $g$ is true) is used.   
\item The transition in which an agent responds to the message which has `?'. For this, the transitions in the 36th row, 37th row and 39th row (which are enabled as the guard $g$ is true) are used.   
\end{itemize}
As a result of the above two observations, if an agent initiates a session 
with the messages $(Init,(x,`?',`?'))$, $(Init,(x,`?',e))$, and $(Init,(x,y,`?')$, 
then the agent will receive the message $(Refute,(x,y',e'))$. Here $y',e'$  may be 
different from $y,e$ respectively. 
\end{myremark}

\noindent
We can construct an abstraction of the set of transitions
in the form of a `message-graph'.
Vertices in the message-graph represent messages sent
or received, and edges are transitions. 
The message graph for the set of transitions from Table \ref{tab:allgt}
is given
in Figure \ref{fig:mgraph}(a): for simplicity, vertex
labels are just the message tags, and edge-labels refer
to the row numbers in the Table \ref{tab:allgt}.\footnote{
Informally, for an edge $(v_1,v_2)$ in the graph, the label for $v_1$ is the message-tag
received, and the label for $v_2$ is the message-tag sent. The edges are labelled
with the corresponding transition entries in the table in Definition \ref{def:trans}.
Correctly, the edge-label should be distinguish between which of $a_n$ or
$a_m$ is sending, and which is receiving along with the message content.
This level of detail is not needed for what follows.} It is evident
from the cycles and self-loops in the graph that
communication can become unbounded. But when we 
restrict to compatible automata, transitions from mentioned
in Proposition \ref{prop:illegal} are not allowed. 
The message graph for
the remaining set of transitions is given in  Figure \ref{fig:mgraph}(b). 
All non-trivial cycles (other than self loops) 
got removed in this message graph. Still the interaction may be unbounded due to
the self-loops. To redress this, we alter
the ${\PXP}$ protocol by replacing transitions
encoding self-loops. 
We will call the modified protocol $\PXPk$.
The corresponding message-graph is in Figure \ref{fig:mgraph}(c).

\begin{mydefinition}[$\PXPk$]
\label{def:lxpstar}
We rename transitions $7,14,16$ and $30$ as ${7}-k,{14}-k,{16}-k$ and ${30}-k$ respectively to denote that at most $k$
occurrences of the transition can occur on any execution; and add the transitions
${7}^\prime, {14}^\prime,{16}^\prime$ and ${30}^\prime$ to allow termination after $k$ iterations. The modified
set of transitions are shown below.

{\scriptsize{
\begin{center}
    \begin{tabular}{|l|c|l|c|c|} \hline
    S.No. & $\mu$ & \multicolumn{1}{|c|}{$P$} & $g$ & $\mu^\prime$ \\ \hline
$7^\prime$. & $(Ratify, (x,y_m,e_m))$  & $H^\prime_n := H_n$  & $g_1$  & $(Term, (x,y^\prime_n,e^\prime_n))$  \\
\hline
$7$-k. & $(Ratify, (x,y_m,e_m))$  & $H^\prime_n := H_n$  & $g_1$  & $(Ratify, (x,y^\prime_n,e^\prime_n))$  \\
\hline$
{14}^\prime$. & $(Refute, (x,y_m,e_m))$  & $H^\prime_n := {\mathtt{LEARN}}(H_n,D^\prime_n)$  & $g_2~\wedge~\neg g^\prime $ & $(Term, (x,y^\prime_n,e^\prime_n))$  \\
\hline
$14$-k. & $(Refute, (x,y_m,e_m))$  & $H^\prime_n :=  {\mathtt{LEARN}}(H_n,D^\prime_n)$  & $g_2~\wedge~\neg g^\prime $ & $(Refute, (x,y^\prime_n,e^\prime_n))$  \\
\hline
 ${16}^\prime$. & $(Refute, (x,y_m,e_m))$  & $H^\prime_n := {\mathtt{LEARN}}(H_n,D^\prime_n)$  & $g_3~\wedge~\neg g^\prime $ & $(Term, (x,y_n,e_n))$  \\
\hline
$16$-k. & $(Refute, (x,y_m,e_m))$  & $H^\prime_n := {\mathtt{LEARN}}(H_n,D^\prime_n)$  & $g_3~\wedge~\neg g^\prime $ & $(Refute, (x,y^\prime_n,e^\prime_n))$  \\
\hline
${30}^\prime$. & $(Reject, (x,y_m,e_m))$  & $H^\prime_n := H_n$  & $g_4$ & $(Term, (x,y^\prime_n,e^\prime_n))$  \\
\hline
$30$-k. & $(Reject, (x,y_m,e_m))$  & $H^\prime_n := H_n$  & $g_4$  & $(Reject, (x,y^\prime_n,e^\prime_n))$  \\
\hline\end{tabular}
\end{center}
}}
\end{mydefinition}

\begin{figure}
    \centering
    \includegraphics[height=8cm]{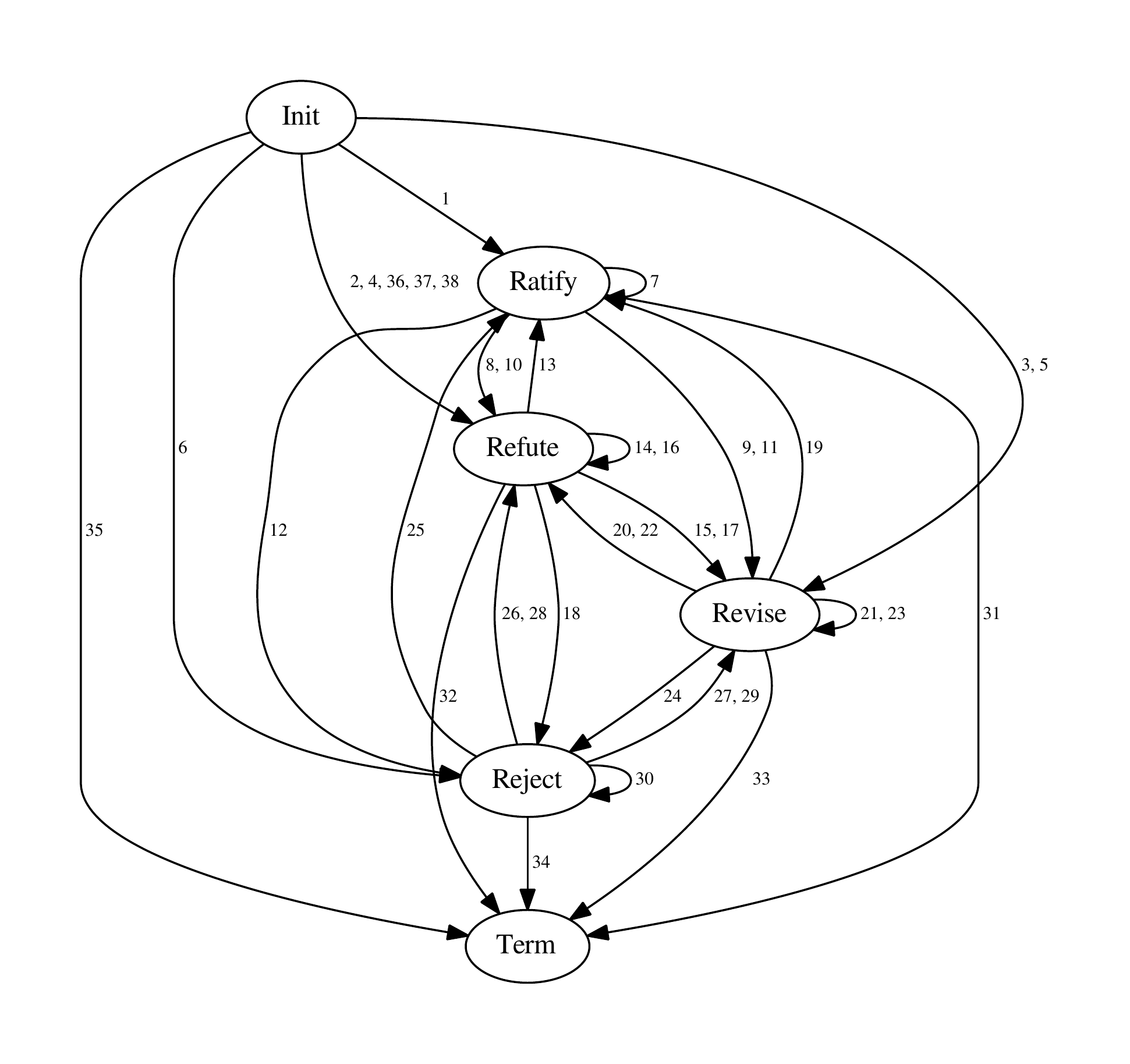}
    \begin{center}
        (a)
    \end{center}
    \vspace*{-.3in}
    \begin{minipage}{0.45\textwidth}
    \centering
    \includegraphics[height=8cm]{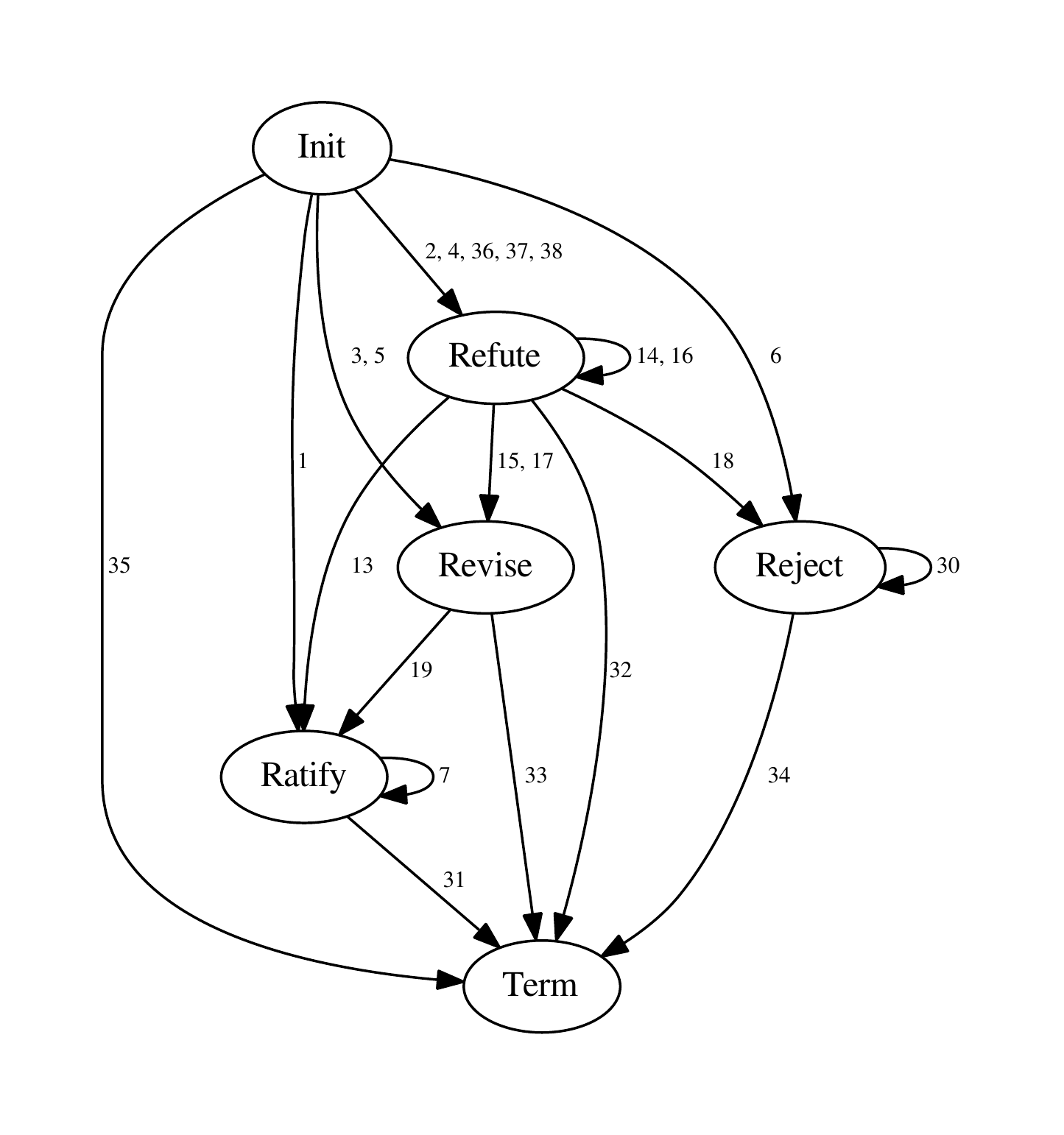}
    \begin{center}
        (b)
    \end{center}
    \end{minipage}
    \hspace*{0.5cm}
    \begin{minipage}{0.40\textwidth}
    \centering
    \vspace*{0.9cm}
    \includegraphics[height=8cm]{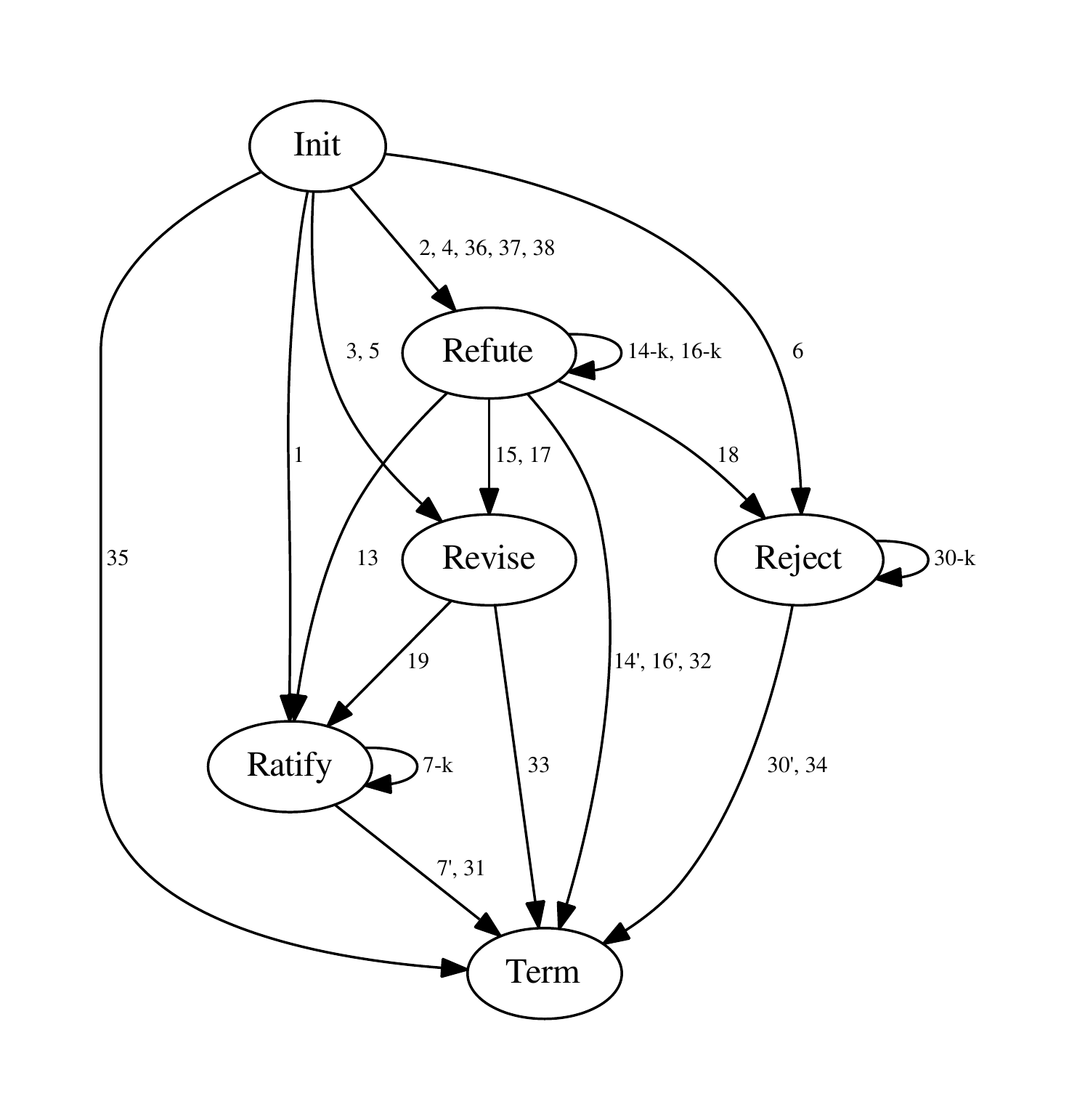}
    \begin{center}
       \hspace*{19pt}
       (c)
    \end{center}
    \vspace*{10pt}
    \end{minipage}
    \caption{
    Message-graph obtained from:
    (a) transitions listed in 
    Table \ref{tab:allgt};
    (b) the transitions in the ${\PXP}$ protocol,
    which is only between compatible agents (transitions mentioned in Proposition \ref{prop:illegal} are excluded) 
    and
    (c) the transitions defined in the ${\PXPk}$ protocol
    (Definition \ref{def:lxpstar}, in which self-loops in ${\PXP}$ are replaced.
    We do not show edges where no message is sent or received.}
        \label{fig:mgraph}
\end{figure}

\subsection{Termination of $\PXPk$}
\label{app:proofterm}

It is straightforward to show that communication between compatible automata using $\PXPk$ is bounded:

\begin{myproposition}[Bounded Communication]
Let Figure \ref{fig:mgraph}(c) represents the message graph of a collaborative session 
using the $\PXPk$ protocol. Then any communication in the session has
bounded length.
\end{myproposition}

\begin{proof}
All messages in a session commence with $Init$ and end with $Term$ message-tags. 
The result follows straightforwardly from the fact that
Figure \ref{fig:mgraph}(c) is a DAG except the self-loops 7-k at Ratify, 14-k and 16-k at Refute, and 30-k at Reject. But each of these loops can occur at most $k$ times.  Therefore any path between $Init$
and $Term$ is bounded. \qed
\end{proof}

\noindent
We note that without the restriction to compatible agents, it is not possible
to guarantee bounded communication (Figure \ref{fig:mgraph}(a), which contains
several cycles).

\section{Completeness of $\PXPk$ wrt to the Dialogue Model {\tt{MMSV}}}
\label{app:madlxp}

The graphical representation of the dialogue model in \citep{madumal} that was shown
Fig.~\ref{fig:madprot}. 

We note that `conversations' arising from execution of the
${\mathtt{MMSV}}$ protocol are paths in the directed-graph in
Figure \ref{fig:madprot}. As an example, Table~\ref{tab:madcompare}(a)
tabulates all paths from Start to End upto length 5 and Table \ref{tab:madcompare}(b)
shows a proposed mapping to message-tag sequences in $\PXPk$.

Two points of difference emerge from Table \ref{tab:madcompare}.
First, there exist message-sequences that are possible
in $\PXPk$ that do not have counterparts in 
${\mathtt{MMSV}}$: an example is
any ${\PXPk}$ sequence containing
$Revise$ for the Explainer. Secondly, there exist paths in ${\mathtt{MMSV}}$
that do not have counterparts in ${\PXPk}$: an
example is any path that contains
 the edge $(1,1)$ in Figure \ref{fig:madprot}.
The former suggests that 
${\mathtt{MMSV}}$, is not intended for use in situations where
either participant in the dialogue can revise its hypothesis
about a data-instance. 
The latter difference arises because questions between $\PEX$ agents
in a session are restricted to the label of the session's instance
and/or the explanation for the label. In such cases,
repeated further question are not meaningful.\footnote{
This is a consequence of assuming that ${\mathtt{MATCH}}$
and ${\mathtt{AGREE}}$ are Boolean functions. If
this assumption does not hold, then an agent $n$ may
not be able to decide whether or not the explanation received from agent $m$ agrees
with the explanation it's own hypothesis derives.
In such a situation the equivalent to further questions
in ${\mathtt{MMSV}}$ become possible. We do not pursue
this further here.} Let us denote ${\mathtt{MMSV}}$ without the edge $(1,1)$ as
${\mathtt{MMSV}}^-$. We note the following:

\begin{table}[htbp]
\begin{center}
{\small{
    \begin{tabular}{|l|l|}\hline
    \multicolumn{1}{|c}{${\mathtt{MMSV}}$ Path} & \multicolumn{1}{|c|}{Path Labels} \\ \hline
    (0,2,8) & [Start]--E:begin-explanation--[Expl. Presented]--end-explanation--[End] \\ \hline
    (0,1,2,8) & [Start]--Q:begin-question--[Ques. Stated]--E:explain/further-explain--\\ 
              & [Expl. Presented]--end-explanation--[End] \\ \hline
    (0,2,3,8) & [Start]--E:begin-explanation--[Expl. Presented]--Q:affirm--\\
              & [Explainee Affirmed]--end-explanation--[End] \\ \hline
    (0,1,2,3,8) & [Start]--Q:begin-question--[Ques. Stated]--E:explain/further-explain--\\
        & [Expl. Presented]--Q:affirm--[Explainee Affirmed]--end-explanation--[End] \\ \hline
    (0,1,1,2,8) & [Start]--Q:begin-question--[Ques. Stated]--E:return-question--\\
        & [Ques. Stated]--E:explain/further-explain--[Expl. Presented]--\\
        & end-explanation--[End] \\ \hline
    (0,2,3,4,8) & [Start]--E:begin-explanation--[Expl. Presented]--Q:affirm--\\
        & [Explainee Affirmed]--E:affirm--[Explainer Affirmed]--end-explanation--[End] \\ \hline
    (0,2,5,6,8) & [Start]--E:begin-explanation--[Expl. Presented]--Q:begin-argument--\\
        & [Arg. Presented]--E:affirm-argument--[Arg. Affirmed]--end-argument--[End] \\ \hline
\end{tabular}
}}
\end{center}
\begin{center}
    (a)
\end{center}
\begin{center}
\begin{tabular}{|l|l|l|}\hline
    \multicolumn{1}{|c}{{$\mathtt{MMSV}$} Path} & \multicolumn{1}{|c|}{${\PXPk}$ Transitions} &\multicolumn{1}{|c|}{${\PXPk}$ Messages} \\ \hline
(0,2,8)& (0,40,39) & $\mathit{Init_E}$,$\mathit{Term_E}$\\ \hline
(0,1,2,8)& (0,36/37/38, 32,39) &$\mathit{Init_Q}$,$\mathit{Refute_E}$,$\mathit{Term_E}$ \\\hline
(0,2,3,8)& (0,1,31,40) & $\mathit{Init_E}$,$\mathit{Ratify_Q}$,$\mathit{Term_E}$\\
& (0,2,32,40) & $\mathit{Init_E}$,$\mathit{Refute_Q}$,$\mathit{Term_E}$\\
& (0,3,33,40) & $\mathit{Init_E}$,$\mathit{Revise_Q}$,$\mathit{Term_E}$\\
\hline
(0,1,2,3,8)& (0,36/37/38,13,39) & $\mathit{Init_Q}$,$\mathit{Refute_E}$,$\mathit{Ratify_Q}$,$\mathit{Term_E}$\\
& (0,36/37/38,15/17,39) & $\mathit{Init_Q}$,$\mathit{Refute_E}$,$\mathit{Revise_Q}$,$\mathit{Term_E}$\\
& (0,36/37/38,14-k/16-k,39) & $\mathit{Init_Q}$,$\mathit{Refute_E}$,$\mathit{Refute_Q}$,$\mathit{Term_E}$\\
\hline
(0,1,1,2,8) & (0,36/37/38,32,39) & $\mathit{Init_Q}$,$\mathit{Refute_E}$,$\mathit{Term_E}$\\
\hline
(0,2,3,4,8) & (0,1/2/3/4/5, & $\mathit{Init_E}$,$\mathit{Ratify_Q/Revise_Q/Refute_Q}$,\\
&7/13/14-k/15/16-k/17/19,39)&  $\mathit{Ratify_E/Revise_E/Refute_E}$,$\mathit{Term_Q}$ \\
\hline
(0,2,5,6,8) & (0,2,13,39) & $\mathit{Init_E}$,$\mathit{Refute_Q}$,$\mathit{Ratify_E}$,$\mathit{Term_Q}$\\
& (0,2,15/17,39) & $\mathit{Init_E}$,$\mathit{Refute_Q}$,$\mathit{Revise_E}$,$\mathit{Term_Q}$\\
& (0,2,14-k/16-k,39) & $\mathit{Init_E}$,$\mathit{Refute_Q}$,$\mathit{Refute_E}$, $\mathit{Term_Q}$\\
\hline
\end{tabular}
\end{center}
\begin{center}
    (b)
\end{center}
    \caption{(a) ${\mathtt{MMSV}}$ paths of up to length 5 from Start to End
        with corresponding node-
        and edge-labels; and (b) Transitions and 
        messages in ${\PXPk}$ corresponding to the ${\mathtt{MMSV}}$ 
        paths in (a). Here the message-tags in the last column with subscripts
        to identify the agent sending the message. We assume $k$ is at least
        2, which allows the ${\PXPk}$ message-sequences
        to contain the 2 consecutive $Refute$ tags for
        the ${\mathtt{MMSV}}$ path $(0,2,3,4,8)$.}
    \label{tab:madcompare}
\end{table}

\begin{myremark}
\begin{itemize}
\item The set of paths in ${\mathtt{MMSV}}^-$ is given by the regular expression
$ 0(1+\epsilon)2(12+312+342+562+5672)^*(8+38+348+568+5678)$.
\item For every path $P$ of  ${\mathtt{MMSV}}^-$, there is a way to decompose path $P$ using
edges in $\{(0,1),(0,2),(1,2)$, $(2,1),(2,3),(2,5),(3,1),(3,4),(5,6),(2,8),
(3,8),(4,8),(6,8),(7,8)\}$ and paths $P_{342},P_{562},P_{567},P_{5672}$ where
$P_{342}$ is a path with edges $(3,4),(4,2)$; $P_{562}$ is a path with edges 
$(5,6),(6,2)$; $P_{567}$ is a path with edges $(5,6),(6,7)$; and $P_{5672}$ is 
a path with edges $(5,6),(6,7),(7,2)$. 
\item Every path of length $l$
in the graph for ${\mathtt{MMSV^-}}$ (Figure \ref{fig:madprot})
corresponds to a path of at most $l$ in the ${\PXP(l)}$ protocol (Figure \ref{fig:mgraph}).
\end{itemize}
\end{myremark}
We will provide a justification for the last point here. We will provide a message tag sequence 
from ${\PXPk}$ protocol for edges in $\{(0,1),(0,2),(1,2),(2,1),
(2,3),(2,5),(3,1),(3,4)$, $(5,6),(2,8)$, $(3,8),(4,8),(6,8),(7,8)\}$ and for paths 
$P_{342},P_{562},P_{567},P_{5672}$. 
\vspace*{-10pt}
\begin{center}
\begin{tabular}{|l|l||l|l|}
\hline
Edge& Message tags & Edge & Message tag\\
\hline
$(0,1)$& $\mathit{Init_Q}$& $(0,2)$ & $\mathit{Init_E}$\\
\hline
$(2,1)$ or $(2,5)$ or & $\mathit{Refute_Q}$& $(1,2)$ & $\mathit{Refute_E}$\\
$(3,1)$&&&\\
\hline
$(2,3)$& $\mathit{Ratify_Q/Revise_Q/Refute_Q}$& $(5,6)$ or $(3,4)$ & $\mathit{Ratify_E/Revise_E/Refute_E}$\\
\hline
$(2,8)$ or $(3,8)$ or & $\mathit{Term_Q/Term_E}$ & $P_{342}$ or $P_{562}$ or  & $\mathit{Refute_E}$\\
$(4,8)$ or $(6,8)$ or &&$P_{567}$ or $P_{5672}$ & \\
$(7,8)$&&&\\
\hline
\end{tabular}
\end{center}
Now for any path $P$ in ${\mathtt{MMSV}}^-$, there is a message sequence in 
${\PXP(l)}$ protocol.

\bibliography{main}

\begin{thebibliography}{41}
\providecommand{\natexlab}[1]{#1}
\providecommand{\url}[1]{{#1}}
\providecommand{\urlprefix}{URL }
\expandafter\ifx\csname urlstyle\endcsname\relax
  \providecommand{\doi}[1]{DOI~\discretionary{}{}{}#1}\else
  \providecommand{\doi}{DOI~\discretionary{}{}{}\begingroup
  \urlstyle{rm}\Url}\fi
\providecommand{\eprint}[2][]{\url{#2}}

\bibitem[{Adamson(2022)}]{Adam:x:2022}
Adamson G (2022) {Ethics and the explainable artificial intelligence (XAI)
  movement}. TechRxiv Preprint techrxiv20439192v1
  \doi{https://doi.org/10.36227/techrxiv.20439192.v1}

\bibitem[{Ai et~al.(2021)Ai, Muggleton, Hocquette, Gromowski, and
  Schmid}]{mugg:expl}
Ai L, Muggleton SH, Hocquette C, Gromowski M, Schmid U (2021) Beneficial and
  harmful explanatory machine learning. Machine Learning 110(4):695--721

\bibitem[{Babic et~al.(2021)Babic, Gerke, Evgeniou, and
  Cohen}]{Babi:etal:j:2021}
Babic B, Gerke S, Evgeniou T, Cohen G (2021) {Beware explanations from AI in
  health care}. Science 373(6552):284--286, \doi{10.1126/science.abg1834}

\bibitem[{Compton(2013)}]{compton2013}
Compton P (2013) Situated cognition and knowledge acquisition research.
  International Journal of Human-Computer Studies 71:184--190,
  \doi{https://doi.org/10.1016/j.ijhcs.2012.10.002}

\bibitem[{Compton and Kang(2021)}]{Comp:Kang:b:2021}
Compton P, Kang B (2021) {Ripple-Down Rules: The Alternative to Machine
  Learning}. CRC Press (Taylor and Francis)

\bibitem[{DARPA(2016)}]{DARP:m:2016}
DARPA (2016) {Explainable Artificial Intelligence (XAI)}. DARPA-BAA-16-53
  \urlprefix\url{https://www.darpa.mil/attachments/DARPA-BAA-16-53.pdf}

\bibitem[{Dash et~al.(2019)Dash, Srinivasan, Joshi, and Baskar}]{dash:drm}
Dash T, Srinivasan A, Joshi RS, Baskar A (2019) {Discrete Stochastic Search and
  Its Application to Feature-Selection for Deep Relational Machines}. In:
  Proceedings of 28th International Conference on Artificial Neural Networks,
  Lecture Notes in Computer Science, vol 11728, pp 29--45

\bibitem[{Dash et~al.(2022)Dash, Srinivasan, and Baskar}]{dash:botgnn}
Dash T, Srinivasan A, Baskar A (2022) {Inclusion of domain-knowledge into GNNs
  using mode-directed inverse entailment}. Mach Learn 111(2):575--623

\bibitem[{Edwards et~al.(1993)Edwards, Compton, Malor, Srinivasan, and
  Lazarus}]{PEIRS:j:1993}
Edwards G, Compton P, Malor R, Srinivasan A, Lazarus L (1993) {PEIRS: A
  pathologist-maintained expert system for the interpretation of chemical
  pathology reports}. Pathology 25(1):27--34

\bibitem[{Guidotti et~al.(2019)Guidotti, Monreale, Ruggieri, Turini, and
  F.~Giannotti}]{Guid:etal:j:2019}
Guidotti R, Monreale A, Ruggieri S, Turini F, F~Giannotti aDP (2019) {A Survey
  of Methods for Explaining Black Box Models}. ACM Computing Surveys
  51(5):1--42, \doi{doi: 10.1145/3236009}

\bibitem[{Gunning and Aha(2019)}]{Gunn:Aha:j:2019}
Gunning D, Aha D (2019) {DARPA's Explainable Artificial Intelligence Program}.
  AI Magazine 40(2):44--58

\bibitem[{Hilton(1990)}]{Hilt:j:1990}
Hilton D (1990) {Conversational Processes and Causal Explanation}.
  Psychological Bulletin 107(1):65--81

\bibitem[{Khincha et~al.(2020)Khincha, Krishnan, Dash, Vig, and
  Srinivasan}]{covid}
Khincha R, Krishnan S, Dash T, Vig L, Srinivasan A (2020) Constructing and
  evaluating an explainable model for covid-19 diagnosis from chest x-rays.
  \urlprefix\url{{https://arxiv.org/abs/2012.10787}}

\bibitem[{King et~al.(2004)King, Whelan, Jones, Reiser, Bryant, Muggleton,
  Kell, and Oliver}]{ross:robot}
King R, Whelan K, Jones F, Reiser P, Bryant C, Muggleton S, Kell D, Oliver S
  (2004) {Functional genomic hypothesis generation and experimentation by a
  robot scientist}. Nature 427:247--252

\bibitem[{Kopec(1982)}]{kopec:thesis}
Kopec D (1982) Human and machine representations of knowledge. PhD thesis,
  University of Edinburgh

\bibitem[{Krenn et~al.(2022)Krenn, Pollice, Guo, Aldeghi, Cervera-Lierta,
  Friederich, dos Passos~Gomes, H{\"a}se, Jinich, Nigam
  et~al.}]{krenn:nature2022}
Krenn M, Pollice R, Guo SY, Aldeghi M, Cervera-Lierta A, Friederich P, dos
  Passos~Gomes G, H{\"a}se F, Jinich A, Nigam A, et~al. (2022) On scientific
  understanding with artificial intelligence. Nature Reviews Physics
  4(12):761--769

\bibitem[{Lipton(2018)}]{Lipt:j:2018}
Lipton Z (2018) {The mythos of model interpretability: In machine learning, the
  concept of interpretability is both important and slippery}. Queue
  16(3):31--57

\bibitem[{Madumal et~al.(2019)Madumal, Miller, Sonenberg, and Vetere}]{madumal}
Madumal P, Miller T, Sonenberg L, Vetere F (2019) A grounded interaction
  protocol for explainable artificial intelligence. In: Proceedings of AAMAS,
  pp 1033--1041, \urlprefix\url{https://arxiv.org/pdf/1903.02409}

\bibitem[{McBurney and Parsons(2002)}]{agentgames}
McBurney P, Parsons S (2002) Games that agents play: A formal framework for
  dialogues between autonomous agents. Journal of Logic, Language and
  Information 11 11:315--334, \doi{https://doi.org/10.1023/A:1015586128739}

\bibitem[{McCarthy(1959)}]{McCa:p:1959}
McCarthy J (1959) Programs with common sense. In: Symposium on Mechanization of
  Thought Processes, National Physical Laboratory, Teddington, England

\bibitem[{McGrath et~al.(2022)McGrath, Kapishnikov, Tomasev, Pearce, Hassabis,
  Kim, Paquet, and Kramnik}]{McGr:etal:x:2022}
McGrath T, Kapishnikov A, Tomasev N, Pearce A, Hassabis D, Kim B, Paquet U,
  Kramnik V (2022) {Acquisition of Chess Knowledge in AlphaZero}. arXiv
  preprint arXiv:211109259 \urlprefix\url{https://arxiv.org/pdf/2111.09259.pdf}

\bibitem[{Michie(1982)}]{michie:window82}
Michie D (1982) Experiments on the mechanization of game-learning. 2-rule-based
  learning and the human window. Comput J 25(1):105--113,
  \doi{10.1093/comjnl/25.1.105},
  \urlprefix\url{https://doi.org/10.1093/comjnl/25.1.105}

\bibitem[{Michie(1988{\natexlab{a}})}]{michie:ewsl88}
Michie D (1988{\natexlab{a}}) Machine learning in the next five years. In:
  Sleeman DH (ed) Proceedings of the Third European Working Session on
  Learning, {EWSL} 1988, Turing Institute, Glasgow, UK, October 3-5, 1988,
  Pitman Publishing, pp 107--122

\bibitem[{Michie(1988{\natexlab{b}})}]{dm:ml}
Michie D (1988{\natexlab{b}}) {M}achine learning in the next five years. In:
  Proceedings of the Third European Working Session on Learning, Pitman, pp
  107--122

\bibitem[{Miller(2019)}]{Mill:j:2019}
Miller T (2019) {Explanation in artificial intelligence: Insights from the
  social sciences}. Artificial Intelligence 267:1--38

\bibitem[{Mozina et~al.(2007)Mozina, Zabkar, and Bratko}]{Mozi:etal:j:2007}
Mozina M, Zabkar J, Bratko I (2007) Argument based machine learning. Artificial
  Intelligence 171:922--937

\bibitem[{Muggleton(1995)}]{mugg:progol}
Muggleton S (1995) {I}nverse {E}ntailment and {P}rogol. New Generation
  Computing 13:245--286

\bibitem[{Muggleton and {De Raedt}(1994)}]{Mug-DeR:j:94}
Muggleton S, {De Raedt} L (1994) {I}nductive {L}ogic {P}rogramming: {T}heory
  and {M}ethods. Journal of {L}ogic {P}rogramming 19(20):629--679

\bibitem[{Plotkin(1971)}]{plotkin:thesis}
Plotkin GD (1971) {A}utomatic {M}ethods of {I}nductive {I}nference. PhD thesis,
  Edinburgh University

\bibitem[{Ray and Moyle(2021)}]{moyle}
Ray O, Moyle S (2021) Towards expert-guided elucidation of cyber attacks
  through interactive inductive logic programming. In: 13th International
  Conference on Knowledge and Systems Engineering, {KSE} 2021, Bangkok,
  Thailand, November 10-12, 2021, {IEEE}, pp 1--7

\bibitem[{Rudin et~al.(2022)Rudin, Chen, Chen, Huang, Semenova, and
  Zhong}]{Rudi:etal:j:2022}
Rudin C, Chen C, Chen Z, Huang H, Semenova L, Zhong C (2022) {Interpretable
  machine learning: Fundamental principles and 10 grand challenges}. Statistics
  Surveys 16:1--85, \doi{https://doi.org/10.1214/21-SS133}

\bibitem[{Sammut and Banerji(1986)}]{sammut:marvin}
Sammut C, Banerji RB (1986) {L}earning concepts by asking questions. In:
  Michalski R, Carbonnel J, Mitchell T (eds) Machine Learning: An Artificial
  Intelligence Approach. Vol. 2, Kaufmann, Los Altos, CA, pp 167--192

\bibitem[{Schmid and Finzel(2020)}]{Schm:Finz:p:2020}
Schmid U, Finzel B (2020) {Mutual Explanations for Cooperative Decision Making
  in Medicine}. In: KI - K{\"{u}}nstliche Intelligenz, Springer, vol~34, pp
  227--233, \doi{http://doi.org/10.1007/s13218-020-00633-2}

\bibitem[{Sokol and Flach(2018)}]{Soko:Flac:p:2018}
Sokol K, Flach P (2018) {Glass-Box: Explaining AI Decisions With Counterfactual
  Statements Through Conversation With a Voice-enabled Virtual Assistant}. In:
  IJCAI 2018: Proceedings of the Twenty-Seventh International Joint Conference
  on Artificial Intelligence, pp 5868--5870

\bibitem[{Sokol and Flach(2021)}]{Soko:Flac:x:2021}
Sokol K, Flach P (2021) Explainability is in the mind of the beholder:
  Establishing the foundations of explainable artificial intelligence. arXiv
  preprint arXiv:211214466

\bibitem[{Srinivasan(2001)}]{aleph}
Srinivasan A (2001) {The Aleph Manual}, oxford University Computing Laboratory

\bibitem[{Stammer et~al.(2021)Stammer, Schramowski, and
  Kersting}]{Stam:etal:p:2021}
Stammer W, Schramowski P, Kersting K (2021) {Right for the right concept:
  Revising neuro-symbolic concepts by interacting with their explanations}. In:
  IEEE Conference on Computer Vision and Pattern Recognition (CVPR 2021), pp
  3619--3629

\bibitem[{Vale et~al.(2022)Vale, El-Sharif, and Ali}]{ValE:etal:j:2022}
Vale D, El-Sharif A, Ali M (2022) {Explainable artificial intelligence (XAI)
  post-hoc explainability methods: risks and limitations in non-discrimination
  law}. AI Ethics \doi{https://doi.org/10.1007/s43681-022-00142-y}

\bibitem[{Wortman~Vaughan and Wallach(2021)}]{jmv:human_ml}
Wortman~Vaughan J, Wallach H (2021) A Human-Centered Agenda for Intelligible
  Machine Learning. MIT Press,
  \urlprefix\url{https://www.microsoft.com/en-us/research/publication/a-human-centered-agenda-for-intelligible-machine-learning/}

\bibitem[{Yeh et~al.(2022)Yeh, Kim, and Ravikumar}]{Yeh:etal:c:2022}
Yeh CK, Kim B, Ravikumar P (2022) {Human-Centered Concept Explanations for
  Neural Networks}. In: Hitzler P, Sarker K (eds) Neuro-Symbolic Artificial
  Intelligence: The State of the Art, Frontiers in Artificial Intelligence and
  Applications, vol 342, IOS Press, pp 337--352

\bibitem[{Zabkar et~al.(2006)Zabkar, Mozina, Videcnik, and
  Bratko}]{Zabk:etal:c:2006}
Zabkar J, Mozina M, Videcnik J, Bratko I (2006) Argument based machine learning
  in a medical domain. In: Computational Models of Argument, IOS Press, pp
  59--70

\end{thebibliography}
\end{document}